\documentclass{article}
\usepackage[main, final]{neurips_2026}

\usepackage[utf8]{inputenc} 
\usepackage[T1]{fontenc}    
\usepackage{hyperref}       
\usepackage{url}            
\usepackage{booktabs}       
\usepackage{amsfonts}       
\usepackage{nicefrac}       
\usepackage{microtype}      
\usepackage{xcolor}         

\usepackage{booktabs}
\usepackage{multirow}
\usepackage[table]{xcolor}
\usepackage{adjustbox}
\usepackage{url}
\usepackage{caption}
\usepackage{amsmath}
\usepackage{amsthm}
\usepackage{algorithm, algorithmic}
\usepackage{enumitem}
\usepackage{tcolorbox}
\usepackage{tabularx}
\usepackage{array}
\usepackage{makecell}
\usepackage{colortbl}
\newcolumntype{Y}{>{\raggedright\arraybackslash}X}

\newcommand{\gray}[1]{\textcolor{gray}{#1}}

\definecolor{lightbestrow}{HTML}{F1F8F2}   
\definecolor{bestrow}{HTML}{E8F5E9}   
\definecolor{darkbestrow}{HTML}{D7EED9}   
\definecolor{headerblue}{HTML}{1565C0} 
\definecolor{lightgray}{HTML}{F5F5F5}

\newcommand{\unmarkedfootnote}[1]{%
  \begingroup
  \renewcommand{\thefootnote}{}%
  \footnotetext{#1}%
  \endgroup
}

\theoremstyle{plain}
\newtheorem{theorem}{Theorem}[section]

\newtheorem{lemma}[theorem]{Lemma}

\theoremstyle{definition}
\newtheorem{definition}[theorem]{Definition}

\theoremstyle{remark}

\title{RelFlexformer: Efficient Attention 3D-Transformers for Integrable Relative Positional Encodings}

\author{Byeongchan Kim\thanks{Equal Contribution. Project page: \url{https://relflexformer.github.io/}}\footnotemark[1]$^{*\color{blue!80!black}{\textbf{1}}}$\quad Arijit Sehanobish\footnotemark[1]$^{*\color{blue!80!black}{\textbf{2}}}$\quad \textbf{Avinava Dubey}\footnotemark[1]$^{*\color{blue!80!black}{\textbf{3}}}$\\
[0.1em]
\textbf{Min-hwan Oh}$^{\color{blue!80!black}{\textbf{1}}}$\quad \textbf{Krzysztof Choromanski}\footnotemark[1]$^{*\color{blue!80!black}{\textbf{4,5}}}$\vspace{0.3cm}\\ 
$^{\color{blue!80!black}{\textbf{1}}}$Seoul National University  $\ \ $ $^{\color{blue!80!black}{\textbf{2}}}$Kensho Technologies  $\ \ $ $^{\color{blue!80!black}{\textbf{3}}}$Google Research  $\ \ $ \\$^{\color{blue!80!black}{\textbf{4}}}$Google DeepMind  $\ \ $ $^{\color{blue!80!black}{\textbf{5}}}$Columbia University
}

\begin{document}
\maketitle

\unmarkedfootnote{Correspondence to:
\href{mailto:bckim97@snu.ac.kr}{bckim97@snu.ac.kr},
\href{mailto:arijit.sehanobish1@gmail.com}{arijit.sehanobish1@gmail.com},
\href{mailto:kchoro@google.com}{kchoro@google.com}.}

\begin{abstract}
We present a new class of efficient attention mechanisms applying universal 3D Relative Positional Encoding (RPE) methods given by arbitrary integrable modulation functions $f$. They lead to the new class of 3D-Transformer models, called \textit{RelFlexformers}, flexibly integrating those RPEs, and characterized by the $O(L \log L)$ time complexity of the attention computation for the $L$-length input sequences. RelFlexformers builds on the theory of the Non-Uniform Fourier Transform (NU-FFT), naturally generalizing several existing efficient RPE-attention methods from structured settings with tokens homogeneously embedded in unweighted grids into general non-structured heterogeneous scenarios, where tokens' positions are arbitrarily distributed in the corresponding 3D spaces. As such, RelFlexformers can be applied in particular to model point clouds. Our extensive empirical evaluation on a large portfolio of 3D datasets confirms quality improvements provided by the NU-FFT-driven attention modulation techniques in the RelFlexformers.
\end{abstract}

\section{Introduction}
\label{sec:intro}

Attention-based Transformers \citep{vaswani17} have become dominant architectures for 3D learning, but their standard form is not aligned well with point cloud (PC) data modality and other geometric data. In these settings, the input tokens are unordered, irregularly spaced in Euclidean space, and often carry strong inductive bias through relative positions~\citep{lu2022transformers}. This has motivated a large family of relative positional encoding (RPE) methods, which inject geometry into attention and improve performance in 3D classification and segmentation~\citep{yang2022one}. At the same time, the quadratic cost of exact attention remains a major bottleneck, especially for large point sets and dense geometric models.

Efficient attention mechanisms, such as Performers~\citep{choromanski2020rethinking}, offer a compelling alternative by replacing the softmax kernel with linear kernels but in the nonlinearly transformed domains, reducing the cost of attention from quadratic to linear time. However, integrating geometric modulation into such models is not straightforward. Unlike standard attention, Performers do not materialize the attention matrix explicitly, so a geometric mask cannot simply be inserted as an element-wise correction. Any viable modulation must instead be compatible with the kernelized formulation and admit fast matrix-vector products~\citep{block-toeplitz}. This requirement is essential: without it, the geometric bias would destroy the computational advantage that makes the architecture useful in the first place.

In this work, we introduce \emph{RelFlexformers}, a class of efficient 3D Transformers that incorporate a broad family of RPEs through an arbitrary integrable modulation function $f$. Our key observation is that the resulting masked attention operator can still be applied efficiently, provided the modulation admits a fast subquadratic matrix-vector multiplication. We provide such an implementation and show that it can be realized in $O(L \log L)$ time for $L$-length sequences, using a \textit{Non-Uniform Fast Fourier Transform} (NU-FFT) formulation. This allows geometric information to be incorporated directly into Performer-style attention without reverting to explicit $L \times L$ attention computation.

RelFlexformers generalize several existing efficient RPE attention methods, typically designed for structured token layouts such as regular grids or homogeneous spatial lattices. In contrast, our formulation applies to arbitrary 3D token configurations, making it well suited to point clouds, depth images (with tokens embedded into 3D spaces, for instance via lifting with the average depth signal across pixels in the patch) and other irregular geometric inputs. The resulting framework is flexible: different modulation functions induce different geometric priors, while the underlying attention computation remains efficient.

We evaluate RelFlexformers on a range of 3D point cloud benchmarks for classification and segmentation. Across these tasks, the proposed NU-FFT-driven modulation consistently improves performance over strong Performer-based baselines and yields a practical route for combining geometric expressivity with subquadratic attention. RelFlexformer is closing the accuracy gap, as compared to the regular quadratic Transformer's attention, often outperforming it.

Our main contributions are: 
\begin{enumerate}
    \item  We formulate a general class of 3D RPE-modulated efficient attention mechanisms based on arbitrary $L_1$-integrable functions. We derive an $O(L \log L)$ implementation via NU-FFT, enabling fast masked attention without materializing the attention matrix (Sec.~\ref{sec:relflexformer}).

    \item We validate RelFlexformer as a versatile, drop-in replacement across diverse backbones (PCT, PTv3, DFormer), consistently elevating baseline Performer attention across synthetic models, sparse LiDAR, and multi-sensor RGB-D data. We also show that RelFlexformer not only bridges the performance gap inherent to efficient linear attention, but actually surpasses standard $O(L^2)$ Transformers on various datasets. Moreover, we show that our distance-based RPE modulation is orthogonal and highly complementary to token-level encodings like PointRoPE and can be combined together (Sec.~\ref{sec:experiments}).
    
\end{enumerate}

\section{Related Work}
\paragraph{Efficient Transformers and Kernel Attention.}
Standard Transformers scale quadratically, $O(L^2)$, with sequence length $L$ \cite{vaswani17}. To mitigate this, one line of research focuses on sparse or local attention \cite{kitaev2020reformer, beltagy2020longformer, zaheer2020bigbird}. Another approach uses kernel-based linear attention to rewrite the softmax operator, achieving $O(L)$ complexity \cite{katharopoulos2020}. Specifically, Performers \cite{choromanski2020rethinking} replace softmax kernels with linear (dot-product) kernels, but in new spaces, obtained via nonlinear transformations of queries and keys. Those mechanisms include in particular techniques applying randomized maps (via random features) to unbiasedly approximate original softmax kernel.
While efficient, these kernelized methods often struggle to integrate structured biases like RPEs without reverting to quadratic cost.

\paragraph{3D Point Cloud Transformers.}
Transformers are naturally suited for 3D point clouds due to their permutation equivariance. Architectures like PCTs~\cite{guo21} and Point Transformers~\cite{zhao2021point} (and their successors) have achieved state-of-the-art results by applying self-attention to local neighborhoods or global point sets. However, these models typically rely on $O(L^2)$ attention or local grouping heuristics to handle geometric relationships. While some models incorporate RPEs via Euclidean distances, they do not provide a mechanism for global, sub-quadratic attention that maintains a general integrable RPE mask. Separately, Point Transformer V3 (PTv3)~\citep{wu2024point} improves scalability by serializing point clouds with space-filling curves and applying dot-product attention within serialized patches. By replacing KNN-based neighborhood construction and explicit RPE with serialized neighbor mapping and sparse-convolutional conditional positional encoding, PTv3 achieves a much larger receptive field with improved speed and memory efficiency. Nevertheless, its efficiency comes from restructuring point clouds into serialization-induced local neighborhoods, rather than from a global sub-quadratic attention mechanism with a general integrable RPE mask. 

\paragraph{Relative Positional Encodings (RPE).}
RPE enhances Transformers by encoding the spatial relationship between tokens \cite{shaw2018self}. Recent advances include ALiBi \cite{alibi}, RoPE \cite{rope} and Fourier Positional Encoding~\cite{fope}, which improve extrapolation. However, most RPE schemes assume a 1D sequence or 2D and 3D grid structures~\cite{schenck2025learning}. In 3D point clouds, the RPE should ideally depend on the continuous distance $f(\mathbf{x}_i - \mathbf{x}_j)$. Luo et al. \cite{luo21} demonstrated that 1D RPE can be computed in $O(L \log L)$ via FFT by exploiting Toeplitz structures. ~\cite{block-toeplitz} provides the first framework to incorporate arbitrary RPE masks in the Performer-style attention, provided the \textit{mask matrix supports fast (sub-quadratic) matrix vector multiplication}. This idea was utilized in~\cite{choromanskifast}, where relative positional encodings (RPE) are implemented using a tree-based mask. The original graph is approximated by a spanning tree, allowing the mask to be applied efficiently via fast matrix–vector multiplication, and thereby integrated into the Performer architecture.~\cite{reid2025linear} uses the graph random features (GRF) to provide inductive biases to linear transformers. These approaches rely on an auxiliary structure, such as a regular grid or a graph/tree, and require additional preprocessing to obtain efficient mask application. In contrast, our goal is to construct the geometric mask itself implicitly for arbitrary non-equispaced 3D coordinates, without materializing a mask or constructing a KNN graph or tree. 

Our work generalizes the Toeplitz based masks on structured grids and runs in subquadratic time. Since point clouds are non-equispaced, we utilize Non-Uniform Fast Fourier Transform (NU-FFT) \cite{greengard2004accelerating, doumeche25} to enable arbitrary 3D RPE masks within a Performer framework, achieving $O(L \log L)$ complexity for irregular geometries. Crucially, this formulation provides a unified geometric operator that subsumes recent multidimensional encodings like STRING \cite{schenck2025learning} as special cases, while extending the same sub-quadratic efficiency to dense RGB-D data via coordinate lifting. We provide a detailed discussion on the evolution of RGB-D representations and their intersection with our unified RPE framework in App. \ref{app:related_works}.

\section{RelFlexformer}
\label{sec:relflexformer}

The RelFlexformer attention mechanism relies on kernelized attention via postive random features~\citep{choromanski2020rethinking} and the theory of the Non-Uniform Fast Fourier Transform (NU-FFT). We provide a basic introduction to Fourier analysis in Sec. \ref{sec:preliminaries} and show how it is applied in RelFlexformer in Sec. \ref{sec:relflexformer-attention}.
\subsection{Preliminaries}
\label{sec:preliminaries}
The Fourier Transform (FT) and Inverse Fourier Transform (IFT) of the $L_{1}$-integrable functions $f,g:\mathbb{R}^{r} \rightarrow \mathbb{R}$ are defined as follows, for the imaginary unit $i \in \mathbb{C}$ ($i^{2}=-1$):
\begin{equation}
\mathcal{F}_{f}(\xi)  = \int_{\mathbb{R}^{r}}\exp(-2\pi i \mathbf{x}^{\top}\xi)f(\mathbf{x})d\mathbf{x}, \textrm{   }
\mathcal{F}^{-1}_{g}(\mathbf{x})  = \int_{\mathbb{R}^{r}}\exp(2\pi i \mathbf{x}^{\top}\xi)g(\xi)d\xi 
\end{equation}
The \textit{Fourier Inversion Theorem} \citep{FT} states that taking the Inverse Fourier Transform of the Fourier Transform of a given function $f$, or vice versa, results in the function $f$ itself:
\begin{equation}
\mathcal{F}^{-1}_{\mathcal{F} f} = \mathcal{F}_{\mathcal{F}^{-1}f} = f.    
\end{equation}

\begin{theorem} [Convolution Theorem \citep{FT}]
Take two functions $g, h: \mathbb{R}^{r} \rightarrow \mathbb{R}$ and their convolution:
\begin{equation}
p(\mathbf{z}) = (g \star h)(\mathbf{z}) \overset{\mathrm{def}}{=} \int_{\mathbb{R}^{r}} g(\tau)h(\mathbf{z}-\tau)d\tau    
\end{equation}
Then the following holds:
\begin{equation}
\mathcal{F}_{p} = \mathcal{F}_{g} \mathcal{F}_{h} \textrm{;   } \mathcal{F}^{-1}_{p} = \mathcal{F}^{-1}_{g} \mathcal{F}^{-1}_{h}   
\end{equation}
\end{theorem}
The FT or IFTs of the convolution of two functions is the product of the FTs or IFTs of those functions. Thus convolving in the original space corresponds to the multiplication in the Fourier space.
Fourier analysis can be extended to objects that are not valid functions, but the so-called \textit{distributions}. One of the most standard examples is a Dirac delta function $\delta_{\mathbf{z}_{0}}$, defined as follows for $\mathbf{z} \in \mathbb{R}^{r}$: 
\begin{equation}
\delta_{\mathbf{z}_{0}}(\mathbf{z}) =
\begin{cases}
0 & \text{if } \mathbf{z} \neq \mathbf{z}_{0} \\
+\infty & \text{if } \mathbf{z} = \mathbf{z}_{0}
\end{cases}
\end{equation}
It can be shown that $\mathcal{F}_{\delta_{\mathbf{z}_{0}}}(\xi) = \exp(-2 \pi i \xi^{\top}\mathbf{z}_{0})$ and $\mathcal{F}^{-1}_{\delta_{\mathbf{z}_{0}}}(\xi) = \exp(2 \pi i \xi^{\top}\mathbf{z}_{0})$.

\begin{definition}[Non-Uniform Fast Fourier Transform (NU-FFT)]
For $a,b \in \mathbb{N}$, two sequences of vectors in $\mathbb{R}^{d}$: $(\alpha_{1},...,\alpha_{a}),(\beta_{1},...,\beta_{b})$ and a sequence of scalars $(c_{1},...,c_{b})$,
the Non-Uniform Fast Fourier Transform is a transformation computing a sequence of scalars: $(\gamma_{1},...,\gamma_{a})$ defined as: 
\begin{equation}
\gamma_{k} = \sum_{j=1}^{b} c_{j} \exp(2 \pi i \alpha_{k}^{\top}\beta_{j})
\end{equation}
\end{definition}
The Non-Uniform FFT can be computed in time $O((a+b)\log(a+b))$.

\textit{Remark :} The above calculations boil down to the multidimensional Discrete Fourier Transform, or DFT (for $d = 1$, the usual DFT) when the vectors $\alpha$ and $\beta$ form a Cartesian grid, i.e. $\beta_j$  must be of the form $\left( \frac{j_1}{N_1}, \frac{j_2}{N_2}, \dots, \frac{j_d}{N_d} \right)$, where each $j_i$ is an integer, and $\alpha_k$ are integer vectors $(k_1, k_2, \dots, k_d)$. Those can be computed efficiently with the use of the standard Fast Fourier Transform (FFT; \citep{FT}).

\subsection{NU-FFT for RPE-modulated attention}
\label{sec:relflexformer-attention}
In this section, we introduce our RPE-attention mechanism, starting from the definition of masked attention. The mask encodes how the RPE mechanism is modulating regular attention computations.

Denote by $\{\mathbf{x}_{i}\}_{i=1}^{L} \subseteq \mathbb{R}^{d}$ the embeddings of input tokens to the attention mechanism and by $\mathbf{X} \in \mathbb{R}^{L \times d}$ a matrix of rows given by $\{\mathbf{x}_{i}\}_{i=1}^{L}$.
We define by $\{\mathbf{q}_{i}\}_{i=1}^{L},\{\mathbf{k}_{i}\}_{i=1}^{L} \subseteq \mathbb{R}^{d_{QK}}, \{\mathbf{v}_{i}\}_{i=1}^{L} \subseteq \mathbb{R}^{d}$ the sets of queries, keys and values respectively, obtained by applying three learnable linear transformations: $\mathbf{W}_{Q} \in \mathbb{R}^{d_{QK} \times d} ,\mathbf{W}_{K}^{d_{QK} \times d},\mathbf{W}_{V} \in \mathbb{R}^{d \times d}$ to embeddings $\{\mathbf{x}_{i}\}_{i=1}^{L}$. 

\begin{tcolorbox}[colback=white, colframe=black, boxrule=0.5pt, arc=2pt, title=\textbf{Algorithm 1: $\mathrm{FastMult}_{\mathbf{M}}$: NU-FFT Mask Multiplication}]
\textbf{Input:} Input vector $\mathbf{u} \in \mathbb{R}^{L}$, token coordinates $\{\mathbf{r}_{i}\}_{i=1}^{L} \in \mathbb{R}^{d}$, spatial modulation function $f$, quadrature samples $\{\xi_{s}\}_{s=1}^{S}$ and coefficients $\{a_{s}\}_{s=1}^{S}$. \\[0.5em]
\textbf{Output:} Vector $\mathbf{w} \in \mathbb{R}^{L}$ approximating $\mathbf{Mu}$.

\vspace{0.2cm}
\begin{enumerate}[label=\textbf{\arabic*.}, leftmargin=*, itemsep=0.15cm]
    \item Compute the Fourier Transform of the point cloud signal at sampled frequencies $s \in \{1, \dots, S\}$ using the forward Non-Uniform FFT in $O(L \log L)$ time: \\
    \hspace*{1.5em} $\mathcal{F}_{P}(\xi_{s}) = \sum_{l=1}^{L} u_{l} \exp(-2\pi i\xi_{s}^{\top} \mathbf{r}_{l})$.
    
    \item Calculate the modulated coefficients for each frequency sample using the spatial mask's Fourier Transform: \\
    \hspace*{1.5em} $b_{s} = a_{s} \mathcal{F}_{f}(\xi_{s}) \mathcal{F}_{P}(\xi_{s})$.
    
    \item Output the final evaluated function at the point coordinates $i \in \{1, \dots, L\}$ using the inverse Non-Uniform FFT in $O(L \log L)$ time: \\
    \hspace*{1.5em} $\mathbf{w}_{i} = \sum_{s=1}^{S} b_{s} \exp(2\pi i\xi_{s}^{\top} \mathbf{r}_{i})$.
\end{enumerate}
\end{tcolorbox}

The unnormalized attention matrix
$\mathbf{A} \in \mathbb{R}^{L \times L}$ is given as: $\mathbf{A} = [\exp(\mathbf{q}_{i}^{\top}\mathbf{k}_{j})]_{i,j \in \{1,...,L\}} \in \mathbb{R}^{L \times L}$. In the regular Relative Positional Encoding (RPE) mechanism, this matrix is then modulated by the so-called \textit{mask matrix} $\mathbf{M} = [f(\mathbf{r}_{i}-\mathbf{r}_{j})]_{i,j \in \{1,...,L\}} \in \mathbb{R}^{L \times L}$, where the set $\{\mathbf{r}_{i}\}_{i=1}^{L} \in \mathbb{R}^{d}$ encodes tokens' positions in some geometric space (e.g. 3D-space for point cloud applications), $f$ is called the \textit{modulation function} and the modulation itself is given via a Hadamard (element-wise) product: $\overline{\mathbf{A}} = \mathbf{A} \odot \mathbf{M}$. Matrix $\overline{\mathbf{A}}$ is then row-wise normalized and the resulting matrix $\overline{\mathbf{A}}_{\textrm{norm}} = [\frac{\overline{\mathbf{A}}_{i,j}}{\sum_{k=1}^{K} \overline{\mathbf{A}}_{i,k}}]$ is applied to the value matrix $\mathbf{V} \in \mathbb{R}^{L \times d_{V}}$, of rows given by $\{\mathbf{v}_{i}\}_{i=1}^{L}$. The final result of attention is given as follows, where the columns of the output matrix are interpreted as new embeddings:
\begin{equation}
\mathbf{X} \rightarrow \mathbf{X} + \overline{\mathbf{A}}_{\textrm{norm}}\mathbf{V}
\end{equation}
Note that the brute-force application of that mechanism requires explicit materialization of the mask matrix $\mathbf{M} \in \mathbb{R}^{L \times L}$, which already leads to quadratic time $O(L^{2})$. To bypass this bottleneck, we build upon linear-attention frameworks such as Performers \citep{choromanski2020rethinking}. Specifically, regular softmax attention $\mathbf{A}_{i,j} = \exp(\mathbf{q}_i^\top \mathbf{k}_j)$ is approximated via a lower-rank factorized counterpart $\tilde{\mathbf{A}}_{i,j} = \phi(\mathbf{q}_{i})^{\top}\phi(\mathbf{k}_{j})$ using a positive random feature map $\phi: \mathbb{R}^{d_{QK}} \rightarrow \mathbb{R}_{+}^{m}$ (e.g., $\phi(\mathbf{x}) = \frac{1}{\sqrt{m}} \exp\left(-\frac{\Vert{}\mathbf{x}\Vert{}^2}{2}\right) \exp(\boldsymbol{\omega}^\top \mathbf{x})$ with Gaussian random samples $\boldsymbol{\omega} \sim \mathcal{N}(\mathbf{0}, \mathbf{I})$). The key advantage of this factorization is that it avoids explicitly materializing the $L\times L$ attention matrix.Instead, by associativity of matrix multiplication, the keys and values can first be aggregated as
$
\sum_{j=1}^{L}\phi(\mathbf{k}_{j})\mathbf{v}_{j}^{\top},
$
after which each query feature $\phi(\mathbf{q}_i)$ is applied to this shared $m\times d_V$ representation. Thus, for fixed feature dimension $m$, the resulting computation scales linearly with $L$ rather than quadratically. 

The positive random feature approximation removes the quadratic cost associated with the unmodulated attention matrix. To incorporate RPEs without reintroducing this quadratic cost, we follow the approach of \citep{block-toeplitz}, which reduces the computation of RPE-modulated linear attention to matrix-vector products with the mask $\mathbf{M}$. Specifically, the resulting complexity is $O(L+T(\mathbf{M}))$, where $T(\mathbf{M})$ denotes the cost of computing $\mathbf{M}\mathbf{u}$ for an arbitrary $\mathbf{u}\in\mathbb{R}^{L}$. Importantly, this formulation does not require explicitly constructing the $L\times L$ mask $\mathbf{M}$. For the relative-position mask
$\mathbf{M}_{ij}=f(\mathbf{r}_{i}-\mathbf{r}_{j})$,
we show that its action on an arbitrary vector can be approximated directly in $O(L\log L)$ time using the NU-FFT. Thus, the central computational problem is to evaluate $\mathbf{M}\mathbf{u}$ efficiently, without ever materializing $\mathbf{M}$.

Take an arbitrary $\mathbf{u} \in \mathbb{R}^{L}$ and $\mathbf{w} = \mathbf{Mu}$. Note that we can re-write $\mathbf{w}_{i}$ for $i=1,...,L$ as:
\begin{equation}
\mathbf{w}_{i} = w(\mathbf{r}_{i}), \textrm{where  } 
w(\mathbf{z}) \overset{\textrm{def}}{=} \int_{\mathbb{R}^{r}} f(\mathbf{z}-\mathbf{k})P(\mathbf{k})d\mathbf{k},
\textrm{   } P(\mathbf{k}) \overset{\textrm{def}}{=} \sum_{l=1}^{L} \mathbf{u}_{l}\delta(\mathbf{k}-\mathbf{r}_{l})
\end{equation}
and $\delta$ is a Dirac delta function. Thus, $w$ as a function of $\mathbf{z}$ can be written as a convolution: $w = f \star P$. 
From Fourier Inversion Theorem and Convolution Theorem (see: Sec. \ref{sec:preliminaries}), we thus get the following:
\begin{equation}
w(\mathbf{z}) = \int_{\mathbb{R}^{r}} \mathcal{F}_{w}(\xi)\exp(2 \pi i \xi^{\top}\mathbf{z})d\xi  
= \int_{\mathbb{R}^{r}} \mathcal{F}_{f}(\xi)\mathcal{F}_{P}(\xi)\exp(2 \pi i \xi^{\top}\mathbf{z})d\xi 
\end{equation}
Using any of the quadrature methods, we then approximate $w$ as:
\begin{equation}
\label{eq:big_sum}
w(\mathbf{z}) \approx \tilde{w}(\mathbf{z}) = \sum_{s=1}^{S} a_{s}\mathcal{F}_{f}(\xi_{s})\mathcal{F}_{P}(\xi_{s})\exp(2\pi i \xi_{s}^{\top}\mathbf{z})     
\end{equation}
for the quadrature coefficients $\{a_{s}\}_{s=1}^{S}$ and samples $\{\xi_{s}\}_{s=1}^{S}$. We will sample $S=O(L)$ points (in practical applications we will even take $S$ to be a constant).
Note that from the linearity of the FT and our observations on the FT of the Dirac delta function from Sec. \ref{sec:preliminaries}, we get:
\begin{equation}
\mathcal{F}_{P}(\xi_{s}) = \sum_{l=1}^{L} \mathbf{u}_{l}\exp(-2\pi i \mathbf{\xi}_{s}^{\top}\mathbf{r}_{l}) \textrm{;    } s=1,...,S.    
\end{equation}
Thus we can compute all $\mathcal{F}_{P}(\xi_{s})$ for $s=1,...,S$ in $O(L\log(L))$ time, applying NU-FFT.
With all those values efficiently calculated, we can define: $b_{s} = a_{s}\mathcal{F}_{f}(\xi_{s})\mathcal{F}_{P}(\xi_{s})$.
Then Equation \ref{eq:big_sum} becomes:
\begin{equation}
w(\mathbf{z}) \approx \tilde{w}(\mathbf{z}) = \sum_{s=1}^{S}b_{s} \exp(2 \pi i \xi_{s}^{\top}\mathbf{z})    
\end{equation}
We need to compute its value in $L$ points: $\mathbf{z}=\mathbf{r}_{1},...,\mathbf{r}_{L}$. We do it again with NU-FFT, again in $O(L \log(L))$ time. That completes the algorithm. Let us call this algorithm  $\textrm{FastMult}_{\mathbf{M}}$.

\textbf{Note:} In practical applications, quadrature coefficients $\{a_{s}\}_{s=1}^{S}$ and samples $\{\xi_{s}\}_{s=1}^{S}$ can also be made learnable. That effectively corresponds to implicitly learning the modulation function $f$.

Now we are ready to integrate this RPE mask $\mathbf{M}$ into Performer-style attention. Recall the low rank attention is given as follows, where $\phi:\mathbb{R}^{d} \rightarrow \mathbb{R}^{m}$ is a nonlinear (potentially randomized) map:
\begin{align}
\begin{split}
    \widehat{\mathrm{Att}_\mathrm{K}} (\mathbf{Q}, \mathbf{K}, \mathbf{V}) = \widehat{\mathbf{D}}^{-1} (\phi(\mathbf{Q})(\phi(\mathbf{K})^{\top} \mathbf{V})), \\
    \quad \widehat{\mathbf{D}} = \mathrm{diag} (\phi(\mathbf{Q})(\phi(\mathbf{K})^{\top} \mathbf{1}_L) ). \label{performers_attention}
\end{split}    
\end{align} 
Here $\mathbf{1}_L$ is the all-ones vector of length $L$, $\mathrm{diag} (\cdot)$ is a diagonal matrix with the input vector as the diagonal
and $\mathbf{Q} \in \mathbb{R}^{L \times d_{QK}}$, $\mathbf{K} \in \mathbb{R}^{L \times d_{QK}}$ and $\mathbf{V} \in \mathbb{R}^{L \times d}$ are the matrices with rows given by the sequences of queries, keys and value vectors respectively.

\citep{block-toeplitz}  investigated how to incorporate the masking in the linear attention as above. Note that in this case $\tilde{\mathbf{A}}$ is never materialized. Building on the work of~\citep{luo21}, the authors of ~\citep{block-toeplitz} proposed a general algorithm that allows for the efficient implementation of masked linear attention. For the convenience of the reader, we present this algorithm in Algorithm 2 box below. We design a mask $\mathbf{M}$ that supports fast matrix vector multiplication, which is exactly the presented above $\textrm{FastMult}_{\mathbf{M}}$ algorithm (Algorithm 1).

\vspace{5mm}

\begin{tcolorbox}[colback=white, colframe=black, boxrule=0.5pt, arc=2pt, title=\textbf{Algorithm 2: General Efficient Low-Rank Masked Attention} \label{alg:mask_att}]
\textbf{Input:} Query/key matrices: $\mathbf{Q},\mathbf{K} \in \mathbb{R}^{L \times d_{QK}}$, value matrix $\mathbf{V} \in \mathbb{R}^{L \times d}$, mask $\mathbf{M} \in \mathbb{R}^{L \times L}$, procedure $\mathrm{FastMult}_{\mathbf{M}}:\mathbb{R}^{L} \rightarrow \mathbb{R}^{L}$ calculating $\mathbf{Mx}$ (or its approximation) for the input $\mathbf{x} \in \mathbb{R}^{L}$, kernel feature map: $\phi:\mathbb{R}^{d_{QK}} \rightarrow \mathbb{R}^{m}$. $\mathrm{vec}(\cdot)$ denotes vectorization. \\[0.5em]
\textbf{Output:} Masked low-rank attention embeddings using $\phi$.

\vspace{0.2cm}
\begin{enumerate}[label=\textbf{\arabic*.}, leftmargin=*, itemsep=0.15cm]
    \item Compute matrices $\mathbf{V}^{1} \in \mathbb{R}^{L \times (md)}$, $\mathbf{V}^{2} \in \mathbb{R}^{L \times m}$ with rows defined as: \\
    \hspace*{1.5em} $\mathbf{V}^{1}_{i:}=\mathrm{vec}(\phi(\mathbf{k}_{i}^{\top})\mathbf{v}_{i})$, $\mathbf{V}^{2}_{i:}=\phi(\mathbf{k}_{i}^{\top})^{\top}$, where $\mathbf{k}_{i}$/$\mathbf{v}_{i}$ stands for the ith row of $\mathbf{K}$/$\mathbf{V}$.
    
    \item $\tilde{\mathbf{D}}^{1} :=[{\color{blue}\mathrm{FastMult}_{\mathbf{M}}(\mathbf{V}^{1}_{:1}),...,\mathrm{FastMult}_{\mathbf{M}}(\mathbf{V}^{1}_{:md})}] \in \mathbb{R}^{L \times md}$, \\
    \hspace*{1.5em} $\tilde{\mathbf{D}}^{2} := [{\color{blue}\mathrm{FastMult}_{\mathbf{M}}(\mathbf{V}^{2}_{:1}),...,\mathrm{FastMult}_{\mathbf{M}}(\mathbf{V}^{2}_{:m})}] \in \mathbb{R}^{L \times m}$ for $\mathbf{V}^{1/2}_{:i}$ denoting ith column of $\mathbf{V}^{1/2}$.
    
    \item Output the embedding $\mathbf{p}_{i}$ of the ith tokens as: \\
    \hspace*{1.5em} $\mathbf{p}_{i} = \frac{\phi(\mathbf{q}_{i}^{\top})^{\top}\mathrm{devec}(\tilde{\mathbf{D}}^{1}_{i:})}{\phi(\mathbf{q}_{i}^{\top})^{\top}(\tilde{\mathbf{D}}^{2}_{i:})^{\top}}$, where $\mathbf{q}_{i}$ is the ith row of $\mathbf{Q}$ and $\mathrm{devec}(\cdot)$ devectorizes its input back to $\mathbb{R}^{m \times d}$.
\end{enumerate}
\end{tcolorbox}

\begin{figure}
    \centering
    \includegraphics[width=\linewidth]{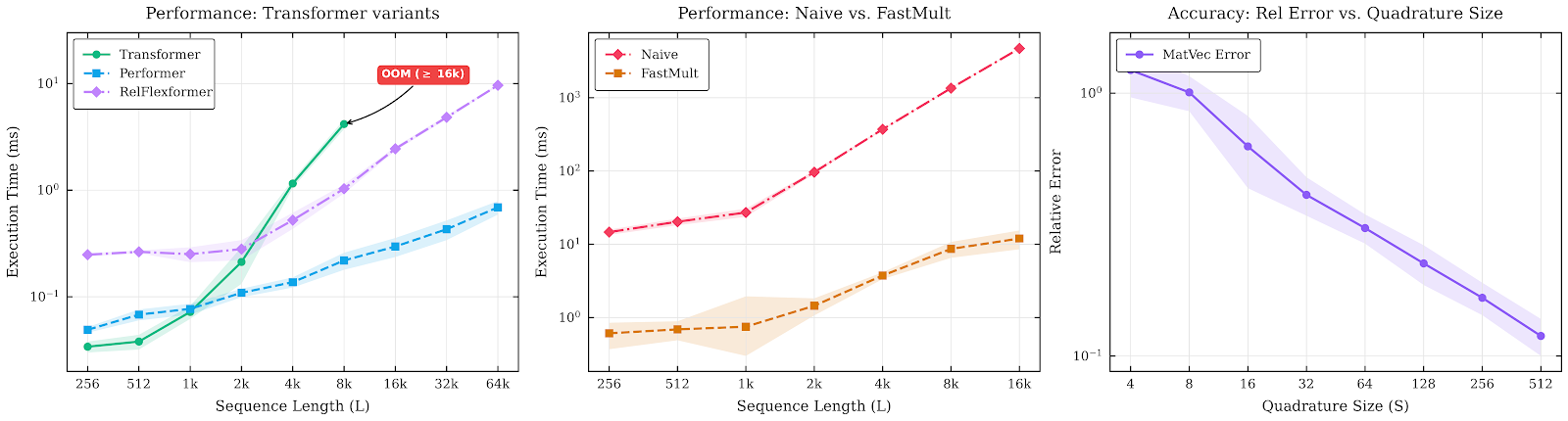}
    \caption{\textbf{Execution time (ms) as a function of sequence length ($L$) and accuracy of our FastMult algorithm.} \textit{Left:} Performance comparison of Transformer, Performer, and RelFlexformer models. Note that the standard Transformer encounters an Out-of-Memory (OOM) error for sequence lengths $L \geq 16k$. \textit{Middle:} Execution time comparison of the Naive approach versus FastMult. \textit{Right:} Relative error of stochastic quadrature approximation to exact mask matrix-vector product as a function of the quadrature size $S$. Error decays as $O(1/\sqrt{S})$.}
    \label{fig:time_plot}
\end{figure}

\
\begin{lemma}[Generalization to RoPE]
For 1D sequential token coordinates, if the continuous spatial modulation function $f(z)$ is parameterized as a set of discrete harmonic functions, the NU-FFT spatial convolution perfectly recovers the discrete rotational shifts of standard Rotary Position Embedding (RoPE).
\end{lemma}

\begin{proof}
Recall that the application of a relative positional mask is computed as a spatial convolution in the frequency domain. For a target point $z$, the masked value $w(z)$ is:
$$ w(z) = \int_{\mathbb{R}} \mathcal{F}_f(\xi)\mathcal{F}_P(\xi)\exp(2\pi i\xi z)d\xi,$$
where $\mathcal{F}_P(\xi) = \sum_{l=1}^L u_l \exp(-2\pi i\xi r_l)$ is the Fourier Transform of the point cloud signal.

In standard RoPE, the dot product between queries and keys for a specific feature dimension $d$ is modulated by a rotation matrix with base frequency $\theta_d$. To replicate this, we define the modulation function for a single feature dimension as $f(z) = \cos(\theta_d z)$. 

Taking the continuous Fourier Transform of this modulation function yields Dirac deltas at the fundamental frequencies:
$$ \mathcal{F}_f(\xi) = \frac{1}{2} \left( \delta\left(\xi - \frac{\theta_d}{2\pi}\right) + \delta\left(\xi + \frac{\theta_d}{2\pi}\right) \right) $$

Substituting this into the convolution integral, by the properties of the Dirac delta, the integral evaluates to zero everywhere except exactly at $\pm \frac{\theta_d}{2\pi}$. The continuous integral thus collapses analytically into a discrete sum:
$$ w(z) = \frac{1}{2} \mathcal{F}_P\left(\frac{\theta_d}{2\pi}\right)\exp(i \theta_d z) + \frac{1}{2} \mathcal{F}_P\left(-\frac{\theta_d}{2\pi}\right)\exp(-i \theta_d z) $$

Consequently, by fixing the quadrature sampling frequencies in Algorithm 1 to $\xi_s \in \left\{\pm \frac{\theta_d}{2\pi}\right\}$, the generalized RelFlexformer algorithm bypasses approximation and analytically evaluates the exact RoPE embedding, all without explicit materialization of the attention matrix.
\end{proof}

\paragraph{Remark (Generalization to STRING~\citep{schenck2025learning}).}
The $\otimes$-STRING position encoding computes a multiplicative mask
$\frac{1}{m}\sum_{k=1}^{m}\cos(\boldsymbol{\omega}_k^\top(\mathbf{r}_i - \mathbf{r}_j))$ with \emph{learnable} frequency vectors $\{\boldsymbol{\omega}_k\}_{k=1}^m$. Our Algorithm~1 computes the same cosine-of-projection form, by drawing $\{\boldsymbol{\omega}_k\}$ from the spectral density $\exp(-\lambda\|\mathbf{x}\|)$
(multivariate Cauchy) and keeps them \emph{fixed}.
Setting $\mathcal{F}_f(\xi_s) = 1$ and replacing the quadrature weights $\{a_s\}$ with uniform $1/S$ recovers $\otimes$-STRING with random frequencies. Conversely, making our frequencies learnable recovers $\otimes$-STRING. Thus STRING is a special case of our framework.

Figure~\ref{fig:time_plot} demonstrates the practical benefits of avoiding explicit materialization of the RPE mask. RelFlexformer remains tractable at sequence lengths where the standard Transformer runs out of memory, while FastMult (via Algorithm 1) substantially reduces the cost of the mask-vector product compared with the naive approach. The stochastic quadrature approximation also exhibits the expected $O(1/\sqrt{S})$ error decay, allowing the computation-accuracy trade-off to be controlled through the quadrature size $S$.
Finally, Fig.~\ref{fig:placeholder} illustrates the behavior of our mask matrix. The mask exhibits a decay pattern similar to standard RBF and Laplace kernels, and its shape can be controlled via different $L_1$ functions and spatial scales.

\begin{figure}
    \centering
    \includegraphics[width=0.85\linewidth]{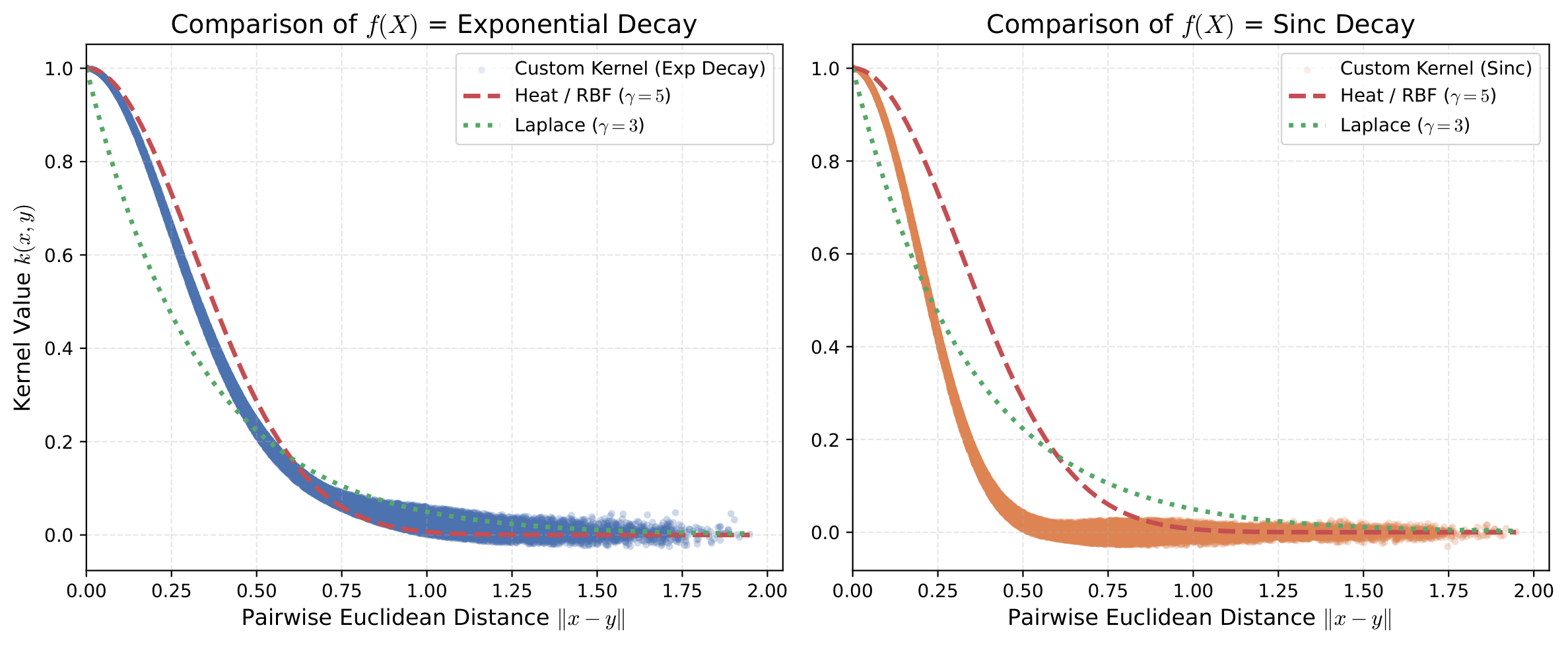}
    \caption{\textit{RelFlexformer} \textbf{Mask Behavior Analysis.} Comparison of our proposed kernels against standard RBF and Laplace kernels. As the Euclidean distance between points increases, the kernel value decays smoothly, mimicking the behavior of Heat and Laplace kernels (dashed lines). The tight variance in the scatter highlights the stability of our projection method across varying spatial scales.}
    \label{fig:placeholder}
\end{figure}

\section{Experiments}
\label{sec:experiments}
In this section, we describe the experimental setup and evaluate the performance of RelFlexformer.
For all experiments, we set $\mathcal{F}_{f}(\xi) := \exp(-\lambda \|\xi\|)$, where $\lambda$ is a modulation parameter.
We incorporate Performer into PCT~\citep{guo21} for object classification (Sec.~\ref{sec:object_classification}), PTv3~\citep{wu2024point} for semantic segmentation (Sec.~\ref{sec:segmentation}) and Dformer for RGB-D data (Sec.~\ref{sec:rgbd}).
We use standard dense self-attention as the baseline and compare it with Performer-based variants that replace dense self-attention with ReLU-Performer attention.
The evaluated variants include vanilla Performer~\citep{choromanski2020rethinking}, Performer with PointRoPE~\citep{yue2026litept}, RelFlexformer, and RelFlexformer with PointRoPE.
PointRoPE~\citep{yue2026litept} extends RoPE to 3D point clouds in a parameter-free manner, allowing us to assess its effect both with Performer and in combination with RelFlexformer. Following~\citep{he2021deberta, t5} insight that attention benefits from position information entering through multiple complementary pathways we combine RoPE (which encodes geometry into token representations) with our RPE (which modulates interaction strength by distance). 

For RGBD data, we also  compare against the recently introduced STRING~\citep{schenck2025learning}. We provide a description for all the datasets used for evaluation in App.~\ref{appendix:datasets}. Finally we provide extensive ablation studies showing only a small quadrature size is optimal for various tasks~\ref{appendix:ablations}.

\subsection{Object Classification}\label{sec:object_classification}
\begin{table}[t]
\caption{Accuracy (\%) comparison on ModelNet40 and Overall and average accuracy for ScanObjectNN classification benchmarks. Bold indicates the best score for each dataset.}
\label{tab:modelnet_scanobjectnn}
\centering
    \begin{tabular}{l|c|cc}
        \toprule
        \multirow{2}{*}{Attention}
        & ModelNet40
        & \multicolumn{2}{c}{ScanObjectNN} \\
        \cmidrule(lr){2-4}
        & OA & mAcc & OA \\
        \midrule
        \gray{Transformer}
        & \gray{93.2}
        & \gray{80.5}  &\gray{84.0}\\
        \midrule

        Performer~\citep{choromanski2020rethinking}
        & 92.34
        & 80.47 & 83.16 \\
        
        ~~+ PointRoPE~\citep{yue2026litept}
        & 92.48
        & 80.36 & 83.88 \\

        \rowcolor{lightbestrow}
        \textbf{RelFlexformer (ours)}
        & \textbf{92.94}
        & \textbf{81.56} & \textbf{84.45} \\
        \rowcolor{bestrow}
        ~~+ PointRoPE
        & 92.55
        &  81.49 & 84.26 \\
        \bottomrule
    \end{tabular}
\end{table}
We evaluate object classification performance of PCT on ModelNet40~\citep{wu20153d} and ScanObjectNN~\citep{uy2019revisiting}. On both ModelNet40 and ScanObjectNN, Performer underperforms the Transformer baseline, indicating a performance gap introduced by efficient attention. Augmenting Performer with PointRoPE yields only marginal gains on both the datasets. RelFlexformer improves performance to \textbf{92.9}\%, reducing most of this gap on ModelNet40. On ScanObjectNN, our method \textit{surpasses the dense Transformer baseline} (see
Tables~\ref{tab:modelnet40} and~\ref{tab:scanobjectnn} for more comparisons with prior methods).

\subsection{RelFlexformer for Semantic Segmentation}\label{sec:segmentation}

\begin{table*}[t]
\caption{Semantic segmentation performance comparison on indoor (ScanNet, ScanNet200, and ScanNet++) and outdoor (nuScenes) benchmarks. Bold indicates the best score for each metric. RelFlexFormer replaces softmax attention (gray) with linear attention plus Fourier RPE, recovering or exceeding dense-attention accuracy at sub-quadratic cost.}
\label{tab:ptv3_sssn}
\centering
    \begin{adjustbox}{width=\textwidth}
        \begin{tabular}{l|ccc|ccc|ccc|ccc}
            \toprule
            \multirow{2}{*}{Attention}
            & \multicolumn{3}{c|}{ScanNet Val}
            & \multicolumn{3}{c|}{ScanNet200 Val}
            & \multicolumn{3}{c|}{ScanNet++ Val}
            & \multicolumn{3}{c}{nuScenes Val} \\
            \cmidrule(lr){2-4}
            \cmidrule(lr){5-7}
            \cmidrule(lr){8-10}
            \cmidrule(lr){11-13}
            & mIoU & mAcc & allAcc
            & mIoU & mAcc & allAcc
            & mIoU & mAcc & allAcc
            & mIoU & mAcc & allAcc \\
            \midrule
            \gray{Transformer}
            & \gray{77.6} & \gray{85.0} & \gray{92.0}
            & \gray{35.3} & \gray{46.0} & \gray{83.4}
            & \gray{48.2} & \gray{61.6} & \gray{87.0}
            & \gray{80.4} & \gray{87.2} & \gray{94.7} \\
            \midrule

            Performer~\citep{choromanski2020rethinking}
            & 74.8 & 83.8 & 91.0
            & 28.2 & 39.6 & 80.2
            & 48.1 & \textbf{62.5} & \textbf{87.2}
            & 72.0 & 80.2 & 93.5 \\
            
            ~~+ PointRoPE~\citep{yue2026litept}
            & 74.6 & 83.4 & 91.0
            & 34.1 & 44.9 & \textbf{82.9}
            & 48.1 & 60.9 & 86.9
            & 80.4 & 87.4 & \textbf{94.8} \\

            \rowcolor{lightbestrow}
            \textbf{RelFlexformer (ours)}
            & \textbf{76.8} & \textbf{85.0} & \textbf{91.9}
            & 34.0 & \textbf{45.4} & \textbf{82.9}
            & 48.7 & 62.2 & 86.8
            & 80.3 & \textbf{87.5} & 94.6 \\

            \rowcolor{bestrow}
            ~~+ PointRoPE~\citep{yue2026litept}
            & 76.6 & 84.5 & 91.6
            & \textbf{34.9} & 45.2 & \textbf{82.9}
            & \textbf{48.8} & 61.7 & 86.7
            & \textbf{81.2} & \textbf{87.5} & \textbf{94.8} \\
            \bottomrule
        \end{tabular}
    \end{adjustbox}
\end{table*}

\begin{table}[t]
\caption{Semantic segmentation performance comparison on S3DIS under 6-fold cross-validation. Bold indicates the best score for each metric. RelFlexFormer matches or exceeds dense-attention performance (gray) on 6-fold cross validation, closing the gap left by Performer alone.}
\label{tab:ptv3_s3dis}
\centering
    \begin{adjustbox}{width=\linewidth}
        \begin{tabular}{llccccccc}
            \toprule
            \textbf{Attention} & \textbf{Metric} & \textbf{Area1} & \textbf{Area2} & \textbf{Area3} & \textbf{Area4} & \textbf{Area5} & \textbf{Area6} & \textbf{6-Fold} \\
            \midrule
            \multirow{3}{*}{\gray{Transformer}}
            & \gray{allAcc} & \gray{93.22} & \gray{86.26} & \gray{94.56} & \gray{90.72} & \gray{91.67} & \gray{94.98} & \gray{91.90} \\
            & \gray{mAcc}   & \gray{89.92} & \gray{74.44} & \gray{94.45} & \gray{81.11} & \gray{78.92} & \gray{93.55} & \gray{85.31} \\
            & \gray{mIoU}   & \gray{83.01} & \gray{63.42} & \gray{86.66} & \gray{71.34} & \gray{73.43} & \gray{87.31} & \gray{77.70} \\
            \midrule
            \multirow{3}{*}{Performer~\citep{choromanski2020rethinking}}
            & allAcc & 92.35 & \textbf{88.53} & 94.47 & 85.96 & 90.86 & 91.20 & 90.56 \\
            & mAcc   & 88.56 & 76.97 & \textbf{93.69} & 78.31 & 75.63 & 85.77 & 83.16 \\
            & mIoU   & 80.74 & 62.60 & 86.02 & 62.60 & 69.75 & 77.35 & 73.18 \\
            \midrule
            \multirow{3}{*}{~~+  PointRoPE~\citep{yue2026litept}}
            & allAcc & 93.00 & 86.88 & 94.25 & 88.29 & 91.17 & 94.73 & 91.39 \\
            & mAcc   & 90.06 & 76.62 & 93.63 & 76.14 & \textbf{77.55} & 93.26 & 84.54 \\
            & mIoU   & 82.31 & 61.47 & 86.21 & 67.38 & 71.75 & 86.59 & 75.95 \\
            \midrule
            \rowcolor{lightbestrow}
            & allAcc & 92.96 & 88.07 & 94.51 & 88.36 & 91.14 & 94.65 & \textbf{91.62} \\
            \rowcolor{lightbestrow}
            & mAcc   & \textbf{90.59} & \textbf{77.33} & 93.00 & 80.92 & 76.79 & \textbf{93.33} & \textbf{85.33} \\
            \rowcolor{lightbestrow}
            \multirow{-3}{*}{\textbf{RelFlexformer (ours)}}
            & mIoU   & 81.93 & \textbf{63.21} & 86.20 & \textbf{68.28} & 71.01 & \textbf{86.88} & 76.25 \\
            \midrule
            \rowcolor{bestrow}
            & allAcc & \textbf{93.32} & 85.84 & \textbf{94.57} & \textbf{88.89} & \textbf{91.50} & \textbf{94.74} & 91.48 \\
            \rowcolor{bestrow}
            & mAcc   & 90.43 & 73.82 & 93.64 & \textbf{81.59} & 77.44 & 93.14 & 85.01 \\
            \rowcolor{bestrow}
            \multirow{-3}{*}{~~+ PointRoPE~\citep{yue2026litept}}
            & mIoU   & \textbf{82.79} & 62.87 & \textbf{86.75} & 68.12 & \textbf{72.14} & 86.79 & \textbf{76.58} \\
            \bottomrule
        \end{tabular}
    \end{adjustbox}
\end{table}
We implement RelFlexformer on top of PTv3~\citep{wu2024point} and follow the training protocol of Pointcept~\citep{pointcept2023}.
Tables~\ref{tab:ptv3_sssn} and~\ref{tab:ptv3_s3dis} compare different attention mechanisms within the PTv3 framework.
All variants are evaluated under the full fine-tuning setting on semantic segmentation benchmarks, including ScanNet~\citep{dai2017scannet}, ScanNet200~\citep{rozenberszki2022language}, ScanNet++~\citep{yeshwanth2023scannet++}, nuScenes~\citep{caesar2020nuscenes}, and S3DIS~\citep{armeni20163d}.
For additional comparisons with prior methods, please see Tables~\ref{tab:scannetv2}--\ref{tab:nuscenes_semseg}.

Table~\ref{tab:ptv3_sssn} shows that replacing Transformer with Performer reduces mIoU across all datasets.
Because Performer~\citep{choromanski2020rethinking} approximates the softmax kernel using random features, it may be less effective at capturing fine-grained spatial interactions in irregular 3D point clouds.
While PointRoPE improves Performer on ScanNet200, nuScenes, and most reported S3DIS areas, the improvements are not consistent across benchmarks.
In particular, it provides no improvement or slightly degrades mIoU on ScanNet, ScanNet++, and S3DIS Area 2.
These results suggest that RoPE-style 3D positional encoding can benefit Performer, while the non-uniform gains highlight the importance of additional relative positional modulation for capturing fine-grained geometric interactions.

RelFlexformer largely resolves this performance drop
 by introducing NU-FFT-driven RPE modulation into efficient attention.
Compared with vanilla Performer, RelFlexformer improves mIoU from 74.8 to \textbf{76.8} on ScanNet, from 28.2 to \textbf{34.0} on ScanNet200, from 48.1 to \textbf{48.7} on ScanNet++, and from 72.0 to \textbf{80.3} on nuScenes.
Table~\ref{tab:ptv3_s3dis} further shows that RelFlexformer improves mIoU over Performer on all six S3DIS held-out areas, with gains ranging from 0.18 to 9.53 points.
The 6-fold mIoU also increases from 73.18 to \textbf{76.25}.
The observed improvements demonstrate that the proposed RPE modulation substantially strengthens Performer-style efficient attention.
It consistently improves over vanilla Performer, closes much of the mIoU gap to dense Transformer attention, and \textit{even surpasses the Transformer on ScanNet++}.

We further examine the interaction between PointRoPE and RelFlexformer.
Across all datasets, RelFlexformer with PointRoPE achieves higher mIoU than Performer with PointRoPE, showing that the proposed RPE modulation provides additional benefits beyond token-level 3D positional encoding.
For example, mIoU increases from 34.1 to \textbf{34.9} on ScanNet200, from 48.1 to \textbf{48.8} on ScanNet++, and from 80.4 to \textbf{81.2} on nuScenes.
Similar improvements are also observed on ScanNet and S3DIS.
As a result, the combined model approaches the dense Transformer on ScanNet, ScanNet200, and most reported S3DIS areas, and \textit{even surpasses it on ScanNet++, nuScenes, and S3DIS Area 3}.

\subsection{RGB-D experiments}\label{sec:rgbd}
In this section, we demonstrate that RelFlexFormer integrates effectively with state-of-the-art architectures such as DFormer~\cite{dformer} for RGB-D semantic segmentation. Starting from a pre-trained DFormer-Base, we replace dense attention with ReLU-Performers and fine-tune on NYU Depth v2~\cite{nyu} and SUN RGB-D~(Table~\ref{tab:main_results}). RelFlexFormer consistently outperforms the Performer baseline and substantially narrows the gap to dense attention. On the challenging SUN RGB-D dataset with multisensor inputs, Performer exhibits a notable performance gap relative to the Transformer. Our method substantially improves performance and effectively closes this gap. Furthermore, we compare against the recently proposed RPE method STRING~\cite{schenck2025learning}. Our approach outperforms STRING while avoiding its additional parameter overhead ($\sim 70$K parameters). For additional comparisons with prior methods, please see Tables~\ref{tab:nyudepthv2} and \ref{tab:sun_rgbd}.
\begin{table}[t]
\centering
\small
\setlength{\tabcolsep}{10pt}
\renewcommand{\arraystretch}{1.15}
\caption{\textbf{RGB-D semantic segmentation results} (mIoU\,\%) on NYU Depth\,v2 and SUN\,RGB-D. All methods use DFormer-Base with HAM decoder. \emph{Performer} replaces softmax with ReLU linear attention. RelFlexformer is the best sub-quadratic variant and substantially closes the gap with dense Transformer.}
\label{tab:main_results}
\vspace{2pt}
\begin{tabular}{@{} l c c c @{}}
\toprule
\textbf{Attention} & \textbf{Complexity} & \textbf{NYU\,v2} & \textbf{SUN\,RGB-D} \\
\midrule
\gray{Softmax}             & \gray{$O(n^2)$}      & \gray{55.6}           & \gray{51.2} \\
Performer                    & $O(n)$        & 54.44          & 48.49 \\
Performer + STRING~\citep{schenck2025learning}           & $O(n\log n)$  & 54.96            & 50.87 \\
\rowcolor{bestrow}
\textbf{RelFlexFormer (ours)} & $O(n\log n)$  & \textbf{55.32} & \textbf{51.04} \\
\bottomrule
\end{tabular}
\end{table}

Additional efficiency benchmarks are provided in Appendix~\ref{app:efficiency}. Across PTv3 and DFormer, RelFlexformer substantially reduces memory and computational cost relative to Softmax attention while retaining relative positional encoding, and exhibits favorable scaling as the number of tokens increases.
\subsection{RPE Improves Dense Attention}
\label{sec:rpe_ablation}
To verify whether performance gains stem from the kernel approximation or positional encoding, we add RelFlexformer RPE directly into standard dense Softmax Transformers. As shown in Table~\ref{tab:rpe_ablation}, introducing our RPE consistently improves dense attention across all datasets, confirming the gains are driven by the \textbf{geometric prior}.
\begin{table}[h]
\centering
\caption{\textbf{Impact of RelFlexformer RPE on dense Transformers.} Evaluated across various 3D modeling datasets, RelFlexFormer improves the dense Transformers.}
\label{tab:rpe_ablation}
\begin{tabular}{l c >{\columncolor{lightbestrow}}c}
\toprule
Dataset & Transformer & \multicolumn{1}{c}{Transformer + RelFlexformer RPE} \\
\midrule
ScanObjectNN (OA)    & 84.0 & \textbf{84.98} \\
ModelNet40 (OA)      & 93.2 & \textbf{93.46} \\
NYU (mIoU, MST)      & 55.6 & \textbf{56.32} \\
SUN RGB-D (mIoU, MST)& 51.2 & \textbf{51.40} \\
ScanNet200 (mIoU)    & 35.3 & \textbf{36.50} \\
\bottomrule
\end{tabular}
\end{table}

In summary, we show that RelFlexFormer significantly improves upon the Performer baseline, closing the gap with dense attention across all datasets and even surpassing Transformers on ScanObjectNN, S3DIS, nuScenes, and ScanNet++. The geometric inductive bias from RelFlexFormer also improves the regular Transformers as well.  Furthermore, when PointRoPE is competitive, combining it with RelFlexFormer yields additional gains, demonstrating that our architecture can effectively integrate multiple complementary inductive biases.

\section{Conclusion}
RelFlexformers provide a new class of efficient
 3D Transformers that combine expressive relative positional encodings with subquadratic attention.
 By framing geometric modulation via the Non-Uniform Fourier Transform, we have shown that it is possible to incorporate a broad family of integrable RPE functions without sacrificing the computational benefits of kernelized attention. Our approach achieves a complexity of $O(L \log L)$, 
 establishing it as a highly scalable architecture for processing large-scale, irregular point clouds
 that were previously bottlenecked by the quadratic cost of standard dense attention. Our empirical results across various 3D benchmarks demonstrate that the flexibility provided by our framework, specifically the ability to handle non-structured, heterogeneous token distributions, leads to consistent performance gains over the baseline Performer and even surpassing the dense Transformers.

\section*{Acknowledgements}
BK and MO were supported by the National Research Foundation of Korea (NRF) grants funded by the Korea government (MSIT)
(No. RS-2022-NR071853, RS-2023-00222663, RS-2026-25613008).

\clearpage
\bibliographystyle{plain}
\bibliography{references}

\newpage

\appendix

\section{Additional Related Work}
\label{app:related_works}
In this section we discuss additional related works mostly focusing on depth images and then discuss differences between our method against other RPE methods. 

\subsection{Related Work on Depth Images}
The integration of depth information into semantic segmentation has evolved from simple heuristic-based approaches to sophisticated geometric modeling. Early methods, such as those by \cite{gupta2014learning}, utilized HHA encoding to map depth into a geocentric coordinate system, enabling the use of 2D CNNs pretrained on RGB datasets. However, these methods often suffered from a "modality gap" where the network failed to learn cross-modal dependencies effectively.

The introduction of the Vision Transformer (ViT) \citep{dosovitskiy2020image} facilitated more flexible fusion strategies. Early ViT-based architectures like TokenFusion \citep{wang2022multimodal} and CMX \citep{zhang2023cmx} focused on cross-modal feature alignment through shared attention mechanisms. Furthermore, hierarchical architectures such as the Swin Transformer \citep{liu2021swin} introduced shifted-window attention to capture multi-scale dependencies, which has become a staple for dense prediction tasks. Despite their success, these models predominantly treat the depth map as a supplementary 2D image, often neglecting the inherent 3D geometric priors of the scene.

A significant paradigm shift occurred with the move toward coordinate lifting, where 2D pixels are unprojected into a 3D camera coordinate space $(x, y, z)$. DFormer~\citep{dformer} and its successor DFormerv2 \citep{yin2025dformerv2} represent milestones in this direction, rethinking RGB-D representation by treating depth maps as geometry priors rather than simple texture maps. By leveraging camera intrinsics ($f_x, f_y, c_x, c_y$), lifting-based methods enable the model to reason about physical scale and spatial proximity. Our work extends this geometric intuition; while recent frameworks like OmniSegmentor \citep{yin2026omnisegmentor} explore flexible multi-modal learning, we argue that capturing the interaction between lifted tokens requires a nuanced understanding of relative spatial distances in 3D space. Unlike 2D-centric models, we operate on a unified geometric representation. This allows our proposed RPE to transcend the 2D pixel grid, effectively unifying the geometric reasoning required for both sparse point cloud representations and dense RGB-D lifting within a single, geometry-agnostic framework.

Relative Positional Encoding (RPE) allows the attention mechanism to attend to the displacement between tokens, a concept popularized in 2D by Swin Transformer \citep{liu2021swin}. However, 2D RPE is typically constrained to discrete indices. Recent work on 3D-aware transformers \citep{zhao2021point, wu2024point} has sought to inject 3D structural integrity into the attention kernel. Most notably, STRING \citep{schenck2025learning} introduced a family of separable, translation-invariant position encodings that generalize Rotary Position Encodings (RoPE) to multidimensional coordinates using a Lie group framework. In this work, we propose a generalized 3D RPE where STRING can be viewed as a specific instantiation. By defining our RPE as a more general operator over Euclidean space, we ensure that the attention mechanism is biased by actual physical distance, which is critical for resolving depth-dependent occlusions and complex 3D layouts across diverse data modalities. 

While 3D-aware geometric reasoning provides superior accuracy, it is inherently limited by the $O(N^2)$ computational complexity of standard attention. This bottleneck is acute in high-resolution RGB-D segmentation. To maintain tractability, we adopt the Performer architecture \citep{choromanski2020rethinking}, utilizing a linear-time approximation of the attention matrix. This allows our model to scale to the high token counts required for dense 3D lifting, achieving a balance of geometric fidelity and efficiency that is often unattainable for standard quadratic Transformers.

\subsection{Comparison with Fast Masked Attention Methods}
\label{app:fast_mask_comparison}

Our masked linear-attention formulation builds on the general observation that a masked attention operator can be evaluated efficiently when the mask admits fast matrix-vector multiplication~\cite{block-toeplitz}. The distinction is in how the mask is represented and evaluated. Prior approaches obtain fast multiplication by exploiting structure in the token domain, whereas RelFlexformer constructs the geometric operator directly from arbitrary non-equispaced coordinates.

\paragraph{Block-Toeplitz masks.}
The Block-Toeplitz formulation~\cite{block-toeplitz} enables fast mask multiplication through FFTs, but its efficiency relies on Toeplitz structure induced by equispaced token locations. This structure is not available for general point clouds, where coordinates are non-equispaced. Applying the same construction therefore requires first mapping the points to a regular grid, which changes the underlying geometry. In contrast, RelFlexformer uses NU-FFT to evaluate the corresponding Fourier representation directly at the input coordinates. When the coordinates lie on a regular grid, our formulation reduces to the conventional Toeplitz/FFT computation as a special case.

\paragraph{Tree-based masks.}
FTFI~\cite{choromanskifast} extends fast masked attention to irregular data by constructing a tree-based representation of the mask. This provides an alternative for non-grid data, but requires explicit construction of the auxiliary tree. This distinction is particularly important for point-cloud Transformers whose token sets can change during training. For example, PCT operates on a subsampled set of 256 points selected from 1024 points using coordinate and feature similarity. Since the selected tokens depend on learned features, the token set changes during optimization, requiring the tree structure to be reconstructed rather than amortized over training. RelFlexformer requires neither a graph nor a tree and operates directly on the current coordinates through NU-FFT.

We measure this difference on the PCT token sets, where FTFI's tree construction requires approximately 14,870 ms per optimization step, compared with 17.4 ms for the RelFlexformer attention computation. Thus, although both methods support irregular token locations, their computational costs differ substantially when the token set changes dynamically. This recurring preprocessing cost is substantial relative to the overall training time of DFormer, which takes approximately 6 hours on NYUv2 and over 24 hours on SUN RGB-D.

\section{Implementation Details}
\label{appendix:implementation_details}
In this section we will describe the datasets and our implementations for the point cloud based tasks and the RGB-D datasets.
To provide a strong and widely used efficient-attention baseline, we additionally implement Performer~\citep{choromanski2020rethinking} across our point cloud and RGB-D experiments. 
Standard dot-product self-attention explicitly computes pairwise interactions between all input tokens, resulting in quadratic complexity with respect to the number of tokens. 
Given query, key, and value matrices $\mathbf{Q}, \mathbf{K}, \mathbf{V} \in \mathbb{R}^{L \times d}$, dense self-attention is computed as

\begin{align}
    \mathrm{Att}(\mathbf{Q}, \mathbf{K}, \mathbf{V}) = \mathbf{D}^{-1}\mathbf{A}\mathbf{V}, \quad \mathbf{A}=\exp(\mathbf{Q}\mathbf{K}^{\top}/\sqrt{d}), \quad \mathbf{D}=\mathrm{diag}(\mathbf{A}\mathbf{1}_L).
\end{align} 

Although this formulation is expressive, the cost of constructing the $L \times L$ attention matrix becomes expensive for dense point clouds and high-resolution RGB-D inputs.

Performer approximates softmax attention using a positive random feature map $\phi:\mathbb{R}^{d} \rightarrow \mathbb{R}^{m}$.

\begin{align}
    \widehat{\mathrm{Att}} (\mathbf{Q}, \mathbf{K}, \mathbf{V}) = \widehat{\mathbf{D}}^{-1} (\phi(\mathbf{Q})(\phi(\mathbf{K})^{\top} \mathbf{V})),
    \quad \widehat{\mathbf{D}} = \mathrm{diag} (\phi(\mathbf{Q})(\phi(\mathbf{K})^{\top} \mathbf{1}_L) ).
\end{align}

This avoids explicitly materializing the full attention matrix and reduces the attention complexity from quadratic to linear in $L$, up to the number of random features. 
Therefore, Performer serves as a suitable baseline for evaluating whether RelFlexformer improves over existing linear-attention mechanisms in 3D perception settings. In our work, $\phi = $ReLU.

For fair comparison, we integrate Performer by replacing only the dense self-attention module in each corresponding backbone, while keeping the remaining architecture, input preprocessing, data augmentation, optimizer, and training schedule unchanged unless otherwise specified. We use fixed projection matrices so our mask adds no additional parameters.

\subsection{Datasets}\label{appendix:datasets}
To evaluate the versatility and scalability of RelFlexformer, we conduct experiments across a diverse set of 3D benchmarks:

\begin{itemize}[leftmargin=2.5em]
    \item \textbf{ModelNet40~\citep{wu20153d}:} A synthetic object classification dataset with 12,311 CAD models from 40 categories, split into 9,843 training and 2,468 testing samples.

    \item \textbf{ScanObjectNN~\citep{uy2019revisiting}:} A real-world object classification dataset with about 15,000 scanned objects across 15 categories, featuring occlusion, partial scans, and background clutter.

    \item \textbf{S3DIS~\citep{armeni20163d}:} A large-scale indoor semantic segmentation dataset with six areas and 13 classes. We follow the standard 6-fold cross-validation protocol.

    \item \textbf{ScanNet v2~\citep{dai2017scannet}:} A large-scale indoor RGB-D dataset with 20 semantic categories. We use the standard split of 1,201 training and 312 validation scans.

    \item \textbf{ScanNet200~\citep{rozenberszki2022language}:} An extension of ScanNet v2 with 200 fine-grained categories, designed to evaluate robustness under long-tailed class distributions.

    \item \textbf{ScanNet++~\citep{yeshwanth2023scannet++}:} A high-fidelity indoor 3D benchmark with accurate geometry, RGB-D captures, and fine-grained semantic annotations.

    \item \textbf{nuScenes~\citep{caesar2020nuscenes}:} A large-scale outdoor autonomous driving dataset with 1,000 urban scenes. We use it to evaluate RelFlexformer on sparse LiDAR point clouds across 16 classes.

    \item \textbf{NYU Depth v2~\citep{nyu}:} 1,449 indoor RGB-D scenes (795 train / 654 test) from a single Kinect v1 sensor, 40 semantic classes, 480×640 resolution.                          
    
    \item \textbf{SUN RGB-D~\citep{sunrgbd}:} 10,335 indoor RGB-D scenes (5,285 train / 5,050 test) from 4 different sensors (Kinect v1/v2, Xtion, RealSense), 37 semantic classes. The multi-sensor variation makes position encoding harder due to varying intrinsics and depth scales.   
\end{itemize}

\subsection{Point Cloud Based Tasks.}
In this section, we will describe our setup and hyperparameters for point cloud classification and segmentation. 

\subsubsection{Classification}

\begin{figure}
    \centering
    \includegraphics[width=\linewidth]{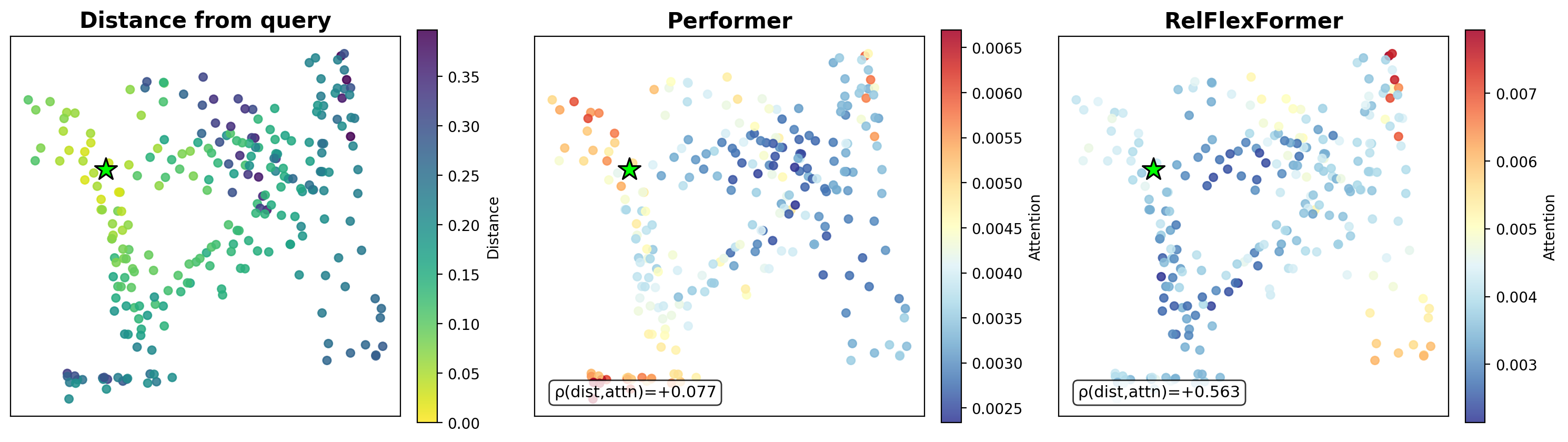}
    \caption{\small{Distance-attention correlation. The left panel shows the spatial distance from the selected query point, while the middle and right panels visualize the corresponding attention values for Performer and RelFlexformer. Compared with Performer, RelFlexformer shows a stronger positive relationship between spatial distance suggesting that its RPE-modulated attention better reflects geometric organization around the query point.}}
    \label{fig:pct_att}
\end{figure}
\paragraph{ModelNet40.}
Our implementation is based on the official PCT~\citep{guo21} codebase.
We evaluate on ModelNet40, which consists of 12{,}311 CAD models from 40 object categories, split into 9{,}843 training and 2{,}468 test objects.
We uniformly sample 1{,}024 points from each object.

We use the PCT classification backbone. The model first applies neighbor embedding based on farthest-point sampling and $k$-NN grouping to capture local geometric structures. The embedded point features are then processed by four stacked efficient attention (Performer) layers with feature dimension $d{=}256$ and 4 attention heads. We concatenate the outputs of the four attention layers and apply a linear projection. Global max pooling and average pooling are then concatenated to obtain an object-level representation, which is passed to an MLP classification head for 40-way classification.

We optimize with SGD using momentum $0.9$ and weight decay $5{\times}10^{-4}$. The initial learning rate is $0.01$ and is scheduled with cosine annealing over 250 epochs. Batch size is 32.

Training augmentations include random point dropout, random rotation around the upright axis, uniform scaling, random shift, jittering, and point shuffling.

For RPE, we sample $S(=8)$ frequency vectors from a random normal distribution at initialization and keep them fixed throughout training. We set $\lambda =1$. All experiments run on NVIDIA H100 80GB HBM3 Tensor Core GPUs. 

Fig.~\ref{fig:pct_att} illustrates RelFlexformer shows a stronger positive relationship between spatial distance and attention patterns than Performer, suggesting that its RPE-modulated attention better reflects geometric organization around the query point.

\paragraph{ScanObjectNN.}
We evaluate on the PB\_T50\_RS variant (hardest split): 2{,}902 training and 581 test objects across 15 classes, consisting of real-world scans with background
clutter and partial views. We subsample 1{,}024 points per object.

We use the PCT backbone: farthest-point sampling with $k$-NN grouping, four self-attention layers ($d{=}256$, 4 heads), global max+average pooling, and an MLP
classification head ($256 \to 256 \to 15$) with dropout $0.5$.

We optimize with SGD (momentum $0.9$, weight decay $5{\times}10^{-4}$) using an initial learning rate of $0.01$ with cosine annealing over 250 epochs. Batch size is 32.

Training augmentations include random point dropout (up to 87.5\%), random $y$-axis rotation, uniform scaling $[0.8, 1.25]$, random shift $\pm 0.1$, jitter
($\sigma{=}0.01$, clip $0.02$), and point shuffle.

For RPE, we sample $S$ frequency vectors from a random normal distribution at
initialization and keep them fixed throughout training. We ablate $S \in \{16, 24, 32, 64, 128, 256\}$ and $\lambda \in \{0.5, 0.75, 1, 1.5, 1.75, 2, 2.25, 3\}$. NVIDIA RTX PRO 6000 (Blackwell) GPUs with 96\,GB VRAM.

\subsubsection{Segmentation}
For experiments based on Pointcept~\citep{pointcept2023}, we use its official PTv3 implementation.
Since our Performer-based variants replace dense Transformer attention with Performer attention, we do not employ FlashAttention, which is designed for dense attention computation.
Accordingly, in the PTv3-based experiments, we set the patch size of both the encoder and decoder to $128$.
Also, Z, TZ, H, and TH denote Z-order, Trans Z-order, Hilbert order, and Trans Hilbert order, respectively.

RelFlexformer is implemented by introducing a continuous relative positional mask into Performer attention.
For all RelFlexformer experiments, we define the Fourier-domain modulation function as $\mathcal{F}_{f}(\xi) := \exp(-\lambda \|\xi\|)$, where $\lambda$ is a modulation parameter.
Unless otherwise specified, the quadrature size $S$ used in the NU-FFT-based approximation of the RelFlexformer relative positional mask is fixed to $8$.

For the PointRoPE variants, we adjust the channel and head configurations to make the per-head feature dimension compatible with the 3D RoPE decomposition.
PointRoPE~\citep{yue2026litept} splits each attention-head feature into three axis-wise subspaces corresponding to the $x$, $y$, and $z$ coordinates and applies 1D RoPE independently to each subspace.
Since RoPE applies pair-wise rotations within each axis, the per-head dimension should be divisible by $6$.
Following the LitePT-S$^{*}$~\citep{yue2026litept} configuration, we set the channel and head dimensions of the PointRoPE attention so that the per-head dimension is $18$.

The detailed configurations are provided in Tables~\ref{tab:model_config_without} and~\ref{tab:model_config_with}.
All experiments run on 4$\times$NVIDIA H100 80GB HBM3 Tensor Core GPUs. We use FP32 training throughout.

\begin{table*}[h]
\centering
\caption{Model configurations for PTv3-based Performer and RelFlexformer variants without PointRoPE.}
\label{tab:model_config_without}
    \centering
    \begin{tabular}{lc}
        \toprule
        \textbf{Config} & \textbf{Value} \\
        \midrule
        \multicolumn{2}{c}{\textit{\textbf{Variants without PointRoPE}}} \\
        \midrule
        Serialization pattern & Z + TZ + H + TH \\
        Patch interaction & Shift Order + Shuffle Order \\
        Positional encoding & xCPE \\
        Embedding depth & 2 \\
        Embedding channels & 32 \\
        Encoder depth & [2, 2, 6, 2] \\
        Encoder channels & [64, 128, 256, 512] \\
        Encoder heads & [4, 8, 16, 32] \\
        Encoder patch size & [128, 128, 128, 128] \\
        Decoder depth & [2, 2, 2, 2] \\
        Decoder channels & [64, 64, 128, 256] \\
        Decoder heads & [4, 4, 8, 16] \\
        Decoder patch size & [128, 128, 128, 128] \\
        Down stride & [$\times$2, $\times$2, $\times$2, $\times$2] \\
        MLP ratio & 4 \\
        QKV bias & True \\
        Drop path & 0.3 \\
        Backbone output channels & 64 \\
        \midrule
        \multicolumn{2}{c}{\textit{\textbf{RelFlexformer-specific}}} \\
        \midrule
        Number of sampled frequencies ($S$) & 8 \\
        Modulation parameter ($\lambda$) & $\{1.0, 2.0, 3.0\}$ \\
        \bottomrule
    \end{tabular}
\end{table*}
\begin{table*}[h]
\centering
\caption{Model configurations for PTv3-based Performer and RelFlexformer variants with PointRoPE.}
\label{tab:model_config_with}
    \centering
    \begin{tabular}{lc}
        \toprule
        \textbf{Config} & \textbf{Value} \\
        \midrule
        \multicolumn{2}{c}{\textit{\textbf{Variants with PointRoPE}}} \\
        \midrule
        Serialization pattern & Z + TZ + H + TH \\
        Patch interaction & Shift Order + Shuffle Order \\
        Positional encoding & xCPE + PointRoPE \\
        Embedding depth & 2 \\
        Embedding channels & 32 \\
        Encoder depth & [2, 2, 6, 2] \\
        Encoder channels & [72, 144, 252, 504] \\
        Encoder heads & [4, 8, 14, 28] \\
        Encoder patch size & [128, 128, 128, 128] \\
        Encoder PointRoPE frequency & [100, 100, 100, 100] \\
        Decoder depth & [2, 2, 2, 2] \\
        Decoder channels & [72, 72, 144, 252] \\
        Decoder heads & [4, 4, 8, 14] \\
        Decoder patch size & [128, 128, 128, 128] \\
        Decoder PointRoPE frequency & [100, 100, 100, 100] \\
        Down stride & [$\times$2, $\times$2, $\times$2, $\times$2] \\
        MLP ratio & 4 \\
        QKV bias & True \\
        Drop path & 0.3 \\
        Backbone output channels & 72 \\
        \midrule
        \multicolumn{2}{c}{\textit{\textbf{RelFlexformer-specific}}} \\
        \midrule
        Number of sampled frequencies ($S$) & 8 \\
        Modulation parameter ($\lambda$) & $\{1.0, 2.0, 3.0\}$ \\
        \bottomrule
    \end{tabular}
\end{table*}

\clearpage
\begin{figure}
    \centering
    \includegraphics[width=\linewidth]{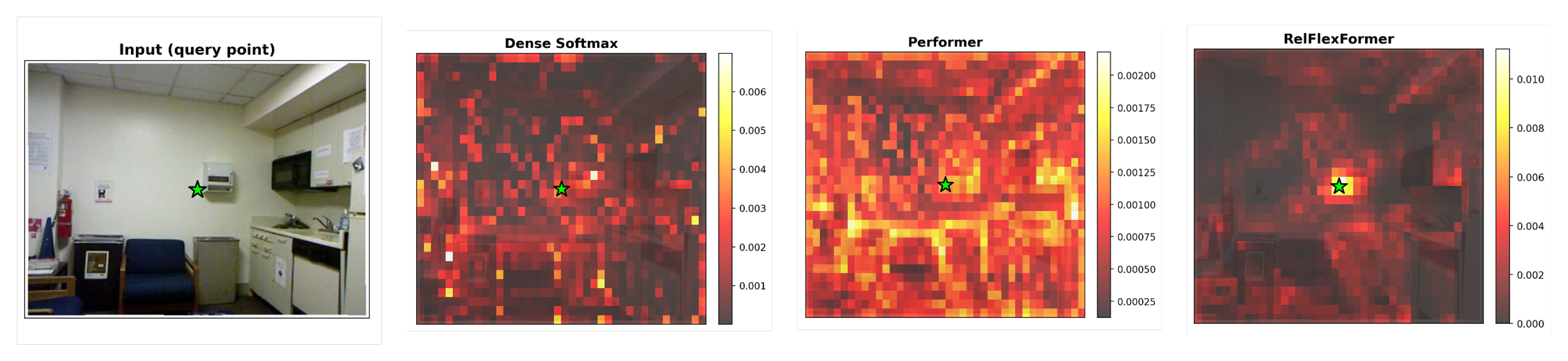}
    \caption{RelFlexformer preserves the efficiency of linear attention while producing more localized, geometry-aware attention patterns in 3D scenes.}
    \label{fig:ptv3_attention}
\end{figure}
\subsection{Segmentation on RGB-D datasets}
All experiments use DFormer-Base as the backbone with a HAM decoder (embedding dim 512). The \emph{Performer} variant replaces softmax attention with ReLU-based linear attention~\cite{choromanski2020rethinking} in the window attention path. \emph{RelFlexFormer} augments Performer attention with a RPE mask. All models initialize from official DFormer-Base pretrained weights (trained with softmax attention). The full backbone is fine-tuned end-to-end.

\paragraph{Hyperparameters.}
Table~\ref{tab:hyperparams} presents the detailed training configurations for RGB-D datasets.

\begin{table}[h]
\centering
\small
\setlength{\tabcolsep}{8pt}
\renewcommand{\arraystretch}{1.15}
\caption{\textbf{Training configuration} for NYU Depth\,v2 and SUN\,RGB-D.}
\label{tab:hyperparams}
\vspace{2pt}
\begin{tabular}{@{} l cc @{}}
\toprule
& \textbf{NYU\,v2} & \textbf{SUN\,RGB-D} \\
\midrule
Optimizer                       & AdamW              & AdamW \\
Learning rate                   & $6{\times}10^{-5}$ & $8{\times}10^{-5}$ \\
LR schedule                     & Poly ($p{=}0.9$)   & Poly ($p{=}0.9$) \\
Weight decay                    & 0.01               & 0.01 \\
Batch size                      & 8                  & 16 \\
Epochs                          & 500                & 300 \\
LR warmup epochs                & 10                 & 10 \\
Drop path rate                  & 0.1                & 0.1 \\
\midrule
\multicolumn{3}{@{}l}{\textit{Fourier RPE}} \\[2pt]
\quad $S$ (frequencies / head)  & 16                 & 16 \\
\quad $\lambda$ (kernel decay)  & 2.0                & 2.0 \\
\quad RPE warmup epochs         & 50 (cosine)        & 30 (cosine) \\
\midrule
\multicolumn{3}{@{}l}{\textit{Evaluation}} \\[2pt]
\quad Crop size                 & $480{\times}640$    & $480{\times}480$ \\
\quad Scales                    & $\{1.0\}$          & $\{1.0\}$ \\
\quad Horizontal flip           & \checkmark         & \checkmark \\
\midrule
\multicolumn{3}{@{}l}{\textit{Precision}} \\[2pt]
\quad Training                  & FP32               & FP32 \\
\quad Validation                & AMP                & AMP \\
\bottomrule
\end{tabular}
\end{table}

\paragraph{Camera Intrinsics.}
NYU Depth\,v2 uses a single Microsoft Kinect\,v1 sensor ($f_x{=}518.86$, $f_y{=}519.47$, $c_x{=}325.58$, $c_y{=}253.74$); all images share identical intrinsics.

SUN\,RGB-D aggregates data from four sensors with varying focal lengths:
\begin{table}[h]
\centering
\small
\setlength{\tabcolsep}{10pt}
\renewcommand{\arraystretch}{1.15}
\caption{\textbf{Sensor distribution} in SUN\,RGB-D.}
\label{tab:sensors}
\vspace{2pt}
\begin{tabular}{@{} l r cc @{}}
\toprule
\textbf{Sensor} & \textbf{Images} & $f_x$ & $f_y$ \\
\midrule
Kinect\,v1      & 1{,}993  & $518.86$ & $519.47$ \\
\rowcolor{lightgray}
Kinect\,v2      & 3{,}684  & 529.5       & 529.5 \\
Asus Xtion      & 3{,}384  & 570.3       & 570.3 \\
\rowcolor{lightgray}
RealSense       & 1{,}152  & 693.7       & 693.7 \\
\bottomrule
\end{tabular}
\end{table}

Per-image intrinsics are loaded from sensor metadata (98.8\% coverage; unmatched images default to Kinect\,v2 values). Pixels are back-projected to 3D world coordinates and normalized per-image to zero mean and unit variance before computing the action of the mask matrix.

\paragraph{RPE Warmup Schedule.}

On NYU, applying the full RPE mask from epoch\,0 yields no improvement over the Performer baseline (54.42 vs.\ 54.44\,mIoU). We attribute this to the RPE mask interfering with the pretrained backbone during early optimization. A warmup schedule resolves this: the RPE mask $M$ is blended from identity to full strength over the first $N$ epochs via
\begin{equation}
\hat{M} = \mathbf{1} + \alpha\,(M - \mathbf{1}), \qquad \alpha = \tfrac{1}{2}\bigl(1 - \cos(\pi\, \min(t/N,\,1))\bigr),
\end{equation}
where $t$ is the current epoch. At $\alpha{=}0$ the model reduces to a pure Performer; at $\alpha{=}1$ the full RPE is applied. With cosine warmup ($N{=}50$ on NYU, $N{=}30$ on SUN), the result improves to \textbf{55.32}\,mIoU on NYU (+0.88 over Performer-only) and \textbf{50.65}\,mIoU on SUN (+2.16 over Performer-only).

All experiments run on NVIDIA RTX PRO 6000 (Blackwell) GPUs with 98\,GB VRAM. We use FP32 training throughout.

\section{Ablation Studies}\label{appendix:ablations}
In this section we will study the effect of the quadrature size $S$ and the modulation parameter $\lambda$.

\subsection{Quadrature Size}
We ablate the quadrature size $S$ across various datasets. A small quadrature suffices for optimal performance.

\begin{figure}[t]
    \centering
    \includegraphics[width=\textwidth]{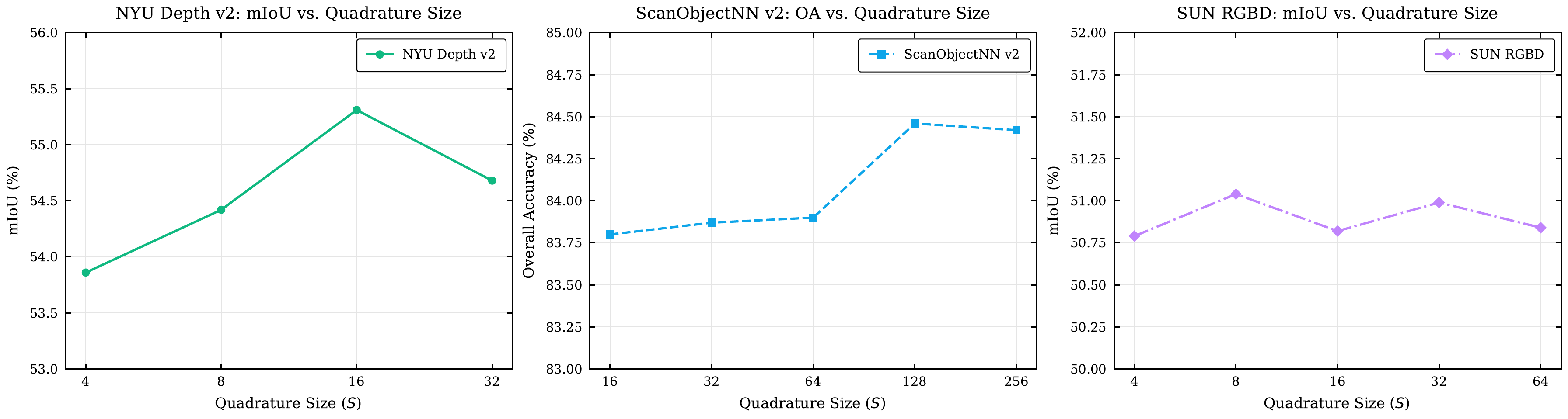}
    \caption{Effect of the quadrature size $S$ per head across different datasets. The plots illustrate the performance trade-offs on NYU Depth v2 (left, mIoU\%), ScanObjectNN v2 (middle, Overall Accuracy\%), and SUN RGB-D (right, mIoU\%) for $\lambda=2$.}
    \label{fig:ablation_S_plots}
\end{figure}
As shown in Fig.~\ref{fig:ablation_S_plots}, increasing $S$ improves performance only up to a task-dependent saturation point, with $\lambda$ fixed to $2.0$. On NYU Depth~v2, the best result is achieved with $S{=}16$, while ScanObjectNN reaches its peak accuracy at $S{=}128$ with almost no further gain from $S{=}256$.

Finally, we show that the quadrature size remains constant even when higher number of points are sampled per point cloud. Table~\ref{tab:ablation_S_scanobj} shows that higher quadrature size is not necessary even when 2048 points are sampled per point cloud.

\begin{table}[h]
\centering
\small
\setlength{\tabcolsep}{12pt}
\renewcommand{\arraystretch}{1.15}
\caption{\small{\textbf{Effect of the size of the quadrature} $S$ per head on ScanObjectNN\,v2 (Overall Accuracy\,\%). $\lambda{=}2$.}}
\label{tab:ablation_S_scanobj}
\vspace{2pt}
\begin{tabular}{@{} l cccccc @{}}
\toprule
& $S{=}16$ & $S{=}32$ & $S{=}64$ & $S{=}128$ & $S{=}256$\\
\midrule
OA & 83.12 & 83.54 & 83.77 & \textbf{84.26} & 84.12  \\
\bottomrule
\end{tabular}
\vspace{-6pt}
\end{table}

\subsection{Modulation parameter}
In this section we ablate over $\lambda$ across various datasets.
\begin{table*}[h]
\caption{Ablation of the modulation parameter $\lambda$ on indoor (ScanNet, ScanNet200, and ScanNet++) and outdoor (nuScenes) semantic segmentation benchmarks.}
\label{tab:ptv3_variants111111111}
\centering
    \begin{adjustbox}{width=\textwidth}
        \begin{tabular}{l|ccc|ccc|ccc|ccc}
            \toprule
            \multirow{2}{*}{Methods}
            & \multicolumn{3}{c|}{ScanNet Val~\citep{dai2017scannet}}
            & \multicolumn{3}{c|}{ScanNet200 Val~\citep{rozenberszki2022language}}
            & \multicolumn{3}{c|}{ScanNet++ Val~\citep{yeshwanth2023scannet++}}
            & \multicolumn{3}{c}{nuScenes Val~\citep{caesar2020nuscenes}} \\
            \cmidrule(lr){2-4}
            \cmidrule(lr){5-7}
            \cmidrule(lr){8-10}
            \cmidrule(lr){11-13}
            & mIoU & mAcc & allAcc
            & mIoU & mAcc & allAcc
            & mIoU & mAcc & allAcc
            & mIoU & mAcc & allAcc \\
            \midrule
            \gray{Transformer}
            & \gray{77.6} & \gray{85.0} & \gray{92.0}
            & \gray{35.3} & \gray{46.0} & \gray{83.4}
            & \gray{48.2} & \gray{61.6} & \gray{87.0}
            & \gray{80.4} & \gray{87.2} & \gray{94.7} \\

            Performer
            & 74.8 & 83.8 & 91.0
            & 28.2 & 39.6 & 80.2
            & 48.1 & 62.5 & 87.2
            & 72.0 & 80.2 & 93.5 \\

            RelFlexformer ($\lambda\!=\!1.0$)
            & 76.9 & 84.8 & 91.8
            & 34.0 & 45.4 & 82.9
            & 48.7 & 62.2 & 86.8
            & 79.9 & 87.3 & 94.5 \\
    
            RelFlexformer ($\lambda\!=\!2.0$)
            & 76.8 & 85.0 & 91.9
            & 33.9 & 44.4 & 82.7
            & 48.3 & 62.4 & 87.1
            & 80.0 & 87.3 & 94.5 \\
            
            RelFlexformer ($\lambda\!=\!3.0$)
            & 75.7 & 84.1 & 91.5
            & 33.8 & 45.2 & 82.7
            & 48.7 & 62.0 & 87.0
            & 80.3 & 87.5 & 94.6 \\
            \bottomrule
        \end{tabular}
    \end{adjustbox}
\end{table*}

\begin{table}[h]
\caption{Ablation of the modulation parameter $\lambda$ on S3DIS 6-fold cross-validation.}
\label{tab:s3dis_6fold111111}
\centering
    \begin{adjustbox}{width=\linewidth}
        \begin{tabular}{llccccccc}
            \toprule
            \textbf{Method} & \textbf{Metric} & \textbf{Area1} & \textbf{Area2} & \textbf{Area3} & \textbf{Area4} & \textbf{Area5} & \textbf{Area6} & \textbf{6-Fold} \\
            \midrule
            \multirow{3}{*}{\gray{Transformer}}
            & \gray{allAcc} & \gray{93.22} & \gray{86.26} & \gray{94.56} & \gray{90.72} & \gray{91.67} & \gray{94.98} & \gray{91.90} \\
            & \gray{mAcc}   & \gray{89.92} & \gray{74.44} & \gray{94.45} & \gray{81.11} & \gray{78.92} & \gray{93.55} & \gray{85.31} \\
            & \gray{mIoU}   & \gray{83.01} & \gray{63.42} & \gray{86.66} & \gray{71.34} & \gray{73.43} & \gray{87.31} & \gray{77.70} \\
            \midrule
            \multirow{3}{*}{Performer}
            & allAcc & 92.35 & 88.53 & 94.47 & 85.96 & 90.86 & 91.20 & 90.56 \\
            & mAcc   & 88.56 & 76.97 & 93.69 & 78.31 & 75.63 & 85.77 & 83.16 \\
            & mIoU   & 80.74 & 62.60 & 86.02 & 62.60 & 69.75 & 77.35 & 73.18 \\
            \midrule
            \multirow{3}{*}{RelFlexformer ($\lambda\!=\!1.0$)}
            & allAcc & 92.96 & 88.07 & 94.51 & 88.36 & 91.01 & 94.75 & 91.61 \\
            & mAcc   & 90.59 & 77.33 & 93.00 & 80.92 & 75.95 & 93.15 & 85.16 \\
            & mIoU   & 81.93 & 63.21 & 86.20 & 68.28 & 70.60 & 86.58 & 76.13 \\
            \midrule
            \multirow{3}{*}{RelFlexformer ($\lambda\!=\!2.0$)}
            & allAcc & 92.88 & 86.69 & 93.84 & 88.16 & 91.08 & 94.65 & 91.22 \\
            & mAcc   & 89.67 & 74.53 & 90.95 & 79.58 & 77.15 & 93.33 & 84.20 \\
            & mIoU   & 81.58 & 62.82 & 83.44 & 67.55 & 71.13 & 86.88 & 75.57 \\
            \midrule
            \multirow{3}{*}{RelFlexformer ($\lambda\!=\!3.0$)}
            & allAcc & 92.72 & 88.43 & 93.88 & 88.03 & 91.14 & 94.48 & 91.45 \\
            & mAcc   & 89.26 & 77.35 & 92.69 & 79.79 & 76.79 & 92.70 & 84.76 \\
            & mIoU   & 81.48 & 62.76 & 84.64 & 66.45 & 71.01 & 86.16 & 75.42 \\
            \bottomrule
        \end{tabular}
    \end{adjustbox}
\end{table}

We study the effect of the modulation parameter $\lambda$ for the exponential decay by evaluating $\lambda \in \{1.0, 2.0, 3.0\}$.

\paragraph{Segmentation on Point Clouds}
RelFlexformer outperforms Performer across all tested values of $\lambda$ on the main benchmark-level mIoU metrics, indicating that the proposed modulation does not rely on a single carefully tuned setting.
Overall, $\lambda=1.0$ gives the best mIoU on ScanNet, ScanNet200, and the S3DIS 6-fold evaluation, while $\lambda=3.0$ performs best on nuScenes and ties with $\lambda=1.0$ on ScanNet++.
For S3DIS Area 5, $\lambda=2.0$ achieves the best mIoU.
This suggests that different scene distributions may benefit from different spatial modulation scales.
Nevertheless, the performance variation across $\lambda$ is relatively small compared with the improvement over Performer, showing that RelFlexformer provides robust gains under different modulation strengths.

While dense PTv3 attention remains a strong reference, RelFlexformer consistently narrows the gap between Performer and Transformer attention and even matches or exceeds the dense-attention reference in some metrics.
These results support our main claim that Performer-style efficient attention benefits from explicit geometric modulation when applied to irregular 3D data.

\paragraph{Classification and RGBD data}
We ablate over $\lambda$ across all three benchmarks (Table~\ref{tab:lambda_ablation_rgbd}). $\lambda=2$ is optimal or near-optimal in these settings, improving over the performer baseline.
\begin{table}[h]
  \centering                                                                             
  \caption{Effect of $\lambda$ across datasets. Best result per dataset in \textbf{bold}.}                              
  \label{tab:lambda_ablation_rgbd}                          
  \setlength{\tabcolsep}{5pt}               
  \begin{tabular}{lccc}                             
  \toprule                                                      
  $\lambda$ & NYU (mIoU) & SUN (mIoU) & ScanObjNN (OA) \\              
  \midrule                                                          
  1 & 54.34 & 48.56 & 83.24 \\                                       
  2 & \textbf{55.32} & \textbf{51.04} & \textbf{84.45} \\            
  3 & 54.08 & 48.18 & 82.79 \\                                    
  \midrule                                                           
  \rowcolor{gray!10}     
  Performer & 54.44 & 48.49 & 83.17 \\     
  \bottomrule              
  \end{tabular}         
\end{table}

\clearpage
\section{Additional Comparisons}
In this section, we provide additional benchmark comparisons for the results reported in Tables~\ref{tab:ptv3_sssn} and~\ref{tab:ptv3_s3dis}.
While the main text focuses on controlled evaluations within the same PTv3 framework, the following tables compare RelFlexformer with representative prior methods reported on the corresponding datasets.
These comparisons are intended to place RelFlexformer in a broader empirical context across commonly reported point-based, convolutional, Transformer-based, Mamba-based, and Performer-based methods.

\subsection{Classification}
\begin{table*}[h]
\centering
\begin{minipage}[t]{0.48\linewidth}
    \centering
    \caption{Classification accuracy on ModelNet40 using 1,024 input points without voting. Bold denotes the best result in the Performer group.}
    \label{tab:modelnet40}
    \small
    \setlength{\tabcolsep}{3pt}
    \renewcommand{\arraystretch}{0.95}
    \begin{tabularx}{\linewidth}{@{}Yc@{}}
        \toprule
        \textbf{Method} & \textbf{OA} \\
        \midrule
        \multicolumn{2}{@{}c@{}}{\textit{\textbf{CNN}}} \\
        \midrule
        PointNet~\citep{qi17}                    & 89.2 \\
        A-SCN~\citep{xie2018attentional}         & 89.9 \\
        PointNet++~\citep{qi2017pointnet++}      & 90.7 \\
        PointGrid~\citep{le2018pointgrid}        & 92.0 \\
        PCNN~\citep{atzmon2018point}             & 92.3 \\
        PointCNN~\citep{li2018pointcnn}          & 92.5 \\
        PointConv~\citep{wu2019pointconv}        & 92.5 \\
        P2Sequence~\citep{liu2019point2sequence} & 92.6 \\
        DGCNN~\citep{wang19}                     & 92.9 \\
        RS-CNN~\citep{liu2019relation}           & 92.9 \\
        PointASNL~\citep{yan2020pointasnl}       & 92.9 \\
        \midrule
        \multicolumn{2}{@{}c@{}}{\textit{\textbf{Transformer}}} \\
        \midrule
        PCT~\citep{guo21}                   & 93.2 \\
        OctFormer~\citep{wang2023octformer} & 92.7 \\
        \midrule
        \multicolumn{2}{@{}c@{}}{\textit{\textbf{Mamba}}} \\
        \midrule
        Point Mamba (w/ OctFormer)~\citep{liu2024point} & 92.7 \\
        Point Mamba (w/ PCT)~\citep{liu2024point}       & 93.4 \\
        PointMamba~\citep{liang2024pointmamba}          & 93.6 \\
        PCM~\citep{zhang2025point}                      & 93.4 \\
        \midrule
        \multicolumn{2}{@{}c@{}}{\textit{\textbf{Performer}}} \\
        \midrule
        Performer (w/ PCT) & 92.3 \\
        \rowcolor{bestrow}
        \textbf{RelFlexformer (w/ PCT) (ours)} & \textbf{92.9} \\
        \rowcolor{bestrow}
        ~~+ PointRoPE~\citep{yue2026litept} & 92.6 \\
        \bottomrule
    \end{tabularx}
\end{minipage}
\hfill
\begin{minipage}[t]{0.48\linewidth}
    \centering
    \caption{Classification accuracy on ScanObjectNN v2 using 1,024 input points, except Point-MAE and PointMamba using 2,048 points, without voting. Bold denotes the best result in the Performer group.}
    \label{tab:scanobjectnn}
    \small
    \setlength{\tabcolsep}{3pt}
    \renewcommand{\arraystretch}{0.95}
    \begin{tabularx}{\linewidth}{@{}Yc@{}}
        \toprule
        \textbf{Method} & \textbf{OA} \\
        \midrule
        \multicolumn{2}{@{}c@{}}{\textit{\textbf{MLP/CNN}}} \\
        \midrule
        PointNet~\citep{qi17}                       & 68.2 \\
        PointNet++~\citep{qi2017pointnet++}         & 77.9 \\
        SpiderCNN~\citep{xu2018spidercnn}           & 73.7 \\
        PointCNN~\citep{li2018pointcnn}             & 78.5 \\
        DGCNN~\citep{wang19}                        & 78.1 \\
        PointMLP~\citep{ma2022rethinking}           & 85.7 \\
        PointNeXt~\citep{qian2022pointnext}         & 87.7 \\
        \midrule
        \multicolumn{2}{@{}c@{}}{\textit{\textbf{Transformer}}} \\
        \midrule
        Point-BERT~\citep{yu2022point}              & 83.1 \\
        MaskPoint~\citep{liu2022masked}             & 84.3 \\
        Point-MAE~\citep{pang2023masked}            & 85.2 \\
        Point-M2AE~\citep{zhang2022point}           & 86.4 \\
        PCT~\citep{guo21} & 84.0 \\
        \midrule
        \multicolumn{2}{@{}c@{}}{\textit{\textbf{Mamba}}} \\
        \midrule
        PointMamba~\citep{liang2024pointmamba}      & 84.9 \\
        PCM-Tiny~\citep{zhang2025point}                 & 86.9 \\
        \midrule
        \multicolumn{2}{@{}c@{}}{\textit{\textbf{Performer}}} \\
        \midrule
        Performer (w/ PCT) & 83.2 \\
        \rowcolor{lightbestrow}
        \textbf{RelFlexformer (w/ PCT) (ours)} & \textbf{84.5} \\
        \rowcolor{bestrow}
        ~~+ PointRoPE~\citep{yue2026litept} & 84.3 \\
        \bottomrule
    \end{tabularx}
\end{minipage}
\end{table*}

\clearpage
\subsection{Segmentation}
\begin{table*}[h]
\centering
\begin{minipage}[t]{0.48\linewidth}
    \centering
    \caption{Semantic segmentation results on ScanNet v2. Val denotes validation mIoU. Bold denotes the best result in the Performer group.}
    \label{tab:scannetv2}
    \small
    \setlength{\tabcolsep}{3pt}
    \renewcommand{\arraystretch}{0.95}
    \begin{tabularx}{\linewidth}{@{}Yc@{}}
        \toprule
        \textbf{Method} & \textbf{Val} \\
        \midrule
        \multicolumn{2}{@{}c@{}}{\textit{\textbf{MLP/CNN}}} \\
        \midrule
        PointNet++~\citep{qi2017pointnet++}        & 53.5 \\
        SparseConvNet~\citep{graham20183d}         & 69.3 \\
        PointConv~\citep{wu2019pointconv}          & 61.0 \\
        JointPointBased~\citep{chiang2019unified}  & 69.2 \\
        KPConv~\citep{thomas2019kpconv}            & 69.2 \\
        PointASNL~\citep{yan2020pointasnl}         & 63.5 \\
        PointNeXt~\citep{qian2022pointnext}        & 71.5 \\
        LargeKernel3D~\citep{chen2023largekernel3d}      & 73.5 \\
        PointMetaBase~\citep{lin2023meta}          & 72.8 \\
        MinkowskiNet~\citep{choy20194d}            & 72.2 \\
        \midrule
        \multicolumn{2}{@{}c@{}}{\textit{\textbf{Transformer}}} \\
        \midrule
        Fast Point Transformer~\citep{park2022fast} & 72.1 \\
        Stratified Transformer~\citep{lai2022stratified} & 74.3 \\
        PointConvFormer~\citep{wu2023pointconvformer} & 74.5 \\
        OctFormer~\citep{wang2023octformer}        & 75.7 \\
        Swin3D~\citep{yang2025swin3d}              & 76.4 \\
        PTv1~\citep{zhao2021point}                 & 70.6 \\
        PTv2~\citep{wu2022point}                   & 75.4 \\
        PTv3~\citep{wu2024point}                   & 77.5 \\
        \midrule
        \multicolumn{2}{@{}c@{}}{\textit{\textbf{Mamba}}} \\
        \midrule
        Point Mamba~\citep{liu2024point}           & 74.6 \\
        Pamba~\citep{li2025pamba}                  & 77.6 \\
        \midrule
        \multicolumn{2}{@{}c@{}}{\textit{\textbf{Performer}}} \\
        \midrule
        Performer (w/ PTv3) & 74.8 \\
        \rowcolor{lightbestrow}
        \textbf{RelFlexformer (w/ PTv3) (ours)} & \textbf{76.9} \\
        \rowcolor{bestrow}
        ~~+ PointRoPE~\citep{yue2026litept} & 76.6 \\
        \bottomrule
    \end{tabularx}
\end{minipage}
\hfill
\begin{minipage}[t]{0.48\linewidth}
    \centering
    \caption{Semantic segmentation results on S3DIS (Area 5). Scores denote mIoU on Area 5. Bold denotes the best result in the Performer group.}
    \label{tab:s3dis_area5}
    \small
    \setlength{\tabcolsep}{3pt}
    \renewcommand{\arraystretch}{0.95}
    \begin{tabularx}{\linewidth}{@{}Yc@{}}
        \toprule
        \textbf{Method} & \textbf{Area 5} \\
        \midrule
        \multicolumn{2}{@{}c@{}}{\textit{\textbf{MLP/CNN}}} \\
        \midrule
        PointNet~\citep{qi17} & 41.1 \\
        SegCloud~\citep{tchapmi2017segcloud} & 48.9 \\
        TanConv~\citep{tatarchenko2018tangent} & 52.6 \\
        PointCNN~\citep{li2018pointcnn} & 57.3 \\
        ParamConv~\citep{wang2018deep} & 58.3 \\
        PointWeb~\citep{zhao2019pointweb} & 60.3 \\
        HPEIN~\citep{jiang2019hierarchical} & 61.9 \\
        KPConv~\citep{thomas2019kpconv} & 67.1 \\
        GACNet~\citep{wang2019graph} & 62.9 \\
        SPGraph~\citep{landrieu2018large} & 58.0 \\
        SegGCN~\citep{lei2020seggen} & 63.6 \\
        PAConv~\citep{xu2021paconv} & 66.6 \\
        PointNeXt~\citep{qian2022pointnext} & 70.5 \\
        PointMetaBase~\citep{lin2023meta} & 72.0 \\
        MinkUNet~\citep{choy20194d} & 65.4 \\
        \midrule
        \multicolumn{2}{@{}c@{}}{\textit{\textbf{Transformer}}} \\
        \midrule
        PAT~\citep{yang2019modeling} & 60.1 \\
        Stratified Transformer~\citep{lai2022stratified} & 72.0 \\
        Superpoint Transformer~\citep{robert2023efficient} & 68.9 \\
        Swin3D~\citep{yang2025swin3d} & 72.5 \\
        PTv1~\citep{zhao2021point} & 70.4 \\
        PTv2~\citep{wu2022point} & 71.6 \\
        PTv3~\citep{wu2024point} & 73.4 \\
        \midrule
        \multicolumn{2}{@{}c@{}}{\textit{\textbf{Mamba}}} \\
        \midrule
        Pamba~\citep{li2025pamba} & 73.5 \\
        \midrule
        \multicolumn{2}{@{}c@{}}{\textit{\textbf{Performer}}} \\
        \midrule
        Performer (w/ PTv3) & 69.8 \\
        \rowcolor{lightbestrow}
        \textbf{RelFlexformer (w/ PTv3) (ours)} & 71.0 \\
        \rowcolor{bestrow}
        ~~+ PointRoPE~\citep{yue2026litept} & \textbf{72.1} \\
        \bottomrule
    \end{tabularx}
\end{minipage}
\end{table*}

\begin{table*}[h]
    \centering
    \caption{Semantic segmentation results on nuScenes. Val denotes validation mIoU. Bold denotes the best result in the Performer group.}
    \label{tab:nuscenes_semseg}
    \small
    \setlength{\tabcolsep}{3pt}
    \renewcommand{\arraystretch}{0.95}
    \begin{tabularx}{0.5\linewidth}{@{}Yc@{}}
        \toprule
        \textbf{Method} & \textbf{Val} \\
        \midrule
        \multicolumn{2}{@{}c@{}}{\textit{\textbf{MLP/CNN/SparseConv}}} \\
        \midrule
        SPVNAS~\citep{tang2020searching} & 77.4 \\
        Cylinder3D~\citep{zhu2021cylindrical} & 76.1 \\
        PVKD~\citep{hou2022point} & 76.0 \\
        2DPASS~\citep{yan20222dpass} & 79.4 \\
        MinkUNet~\citep{choy20194d} & 73.3 \\
        ~~+ M3Net~\citep{liu2024multi} & 79.0 \\
        ~~+ PPT~\citep{wu2024towards} & 78.6 \\
        OA-CNNs~\citep{peng2024oa} & 78.9 \\
        \midrule
        \multicolumn{2}{@{}c@{}}{\textit{\textbf{Transformer}}} \\
        \midrule
        SphereFormer~\citep{lai2023spherical} & 78.4 \\
        RangeFormer~\citep{kong2023rethinking} & 78.1 \\
        PTv2~\citep{wu2022point} & 80.2 \\
        PTv3~\citep{wu2024point} & 80.4 \\
        \midrule
        \multicolumn{2}{@{}c@{}}{\textit{\textbf{Mamba}}} \\
        \midrule
        Pamba~\citep{li2025pamba} & 80.4 \\
        \midrule
        \multicolumn{2}{@{}c@{}}{\textit{\textbf{Performer}}} \\
        \midrule
        Performer (w/ PTv3) & 72.0 \\
        \rowcolor{lightbestrow}
        \textbf{RelFlexformer (w/ PTv3) (ours)} & 80.3 \\
        \rowcolor{bestrow}
        ~~+ PointRoPE~\citep{yue2026litept} & \textbf{81.2} \\
        \bottomrule
    \end{tabularx}
\end{table*}

\clearpage
\subsection{Segmentation on RGB-D datasets}
\begin{table*}[h]
    \centering
    \caption{Segmentation results on NYU Depth v2. Bold denotes the best result in the Performer group.}
    \label{tab:nyudepthv2}
    \small
    \setlength{\tabcolsep}{3pt}
    \renewcommand{\arraystretch}{0.95}
    \begin{tabularx}{0.85\linewidth}{@{}Ylrc@{}}
        \toprule
        \textbf{Method} & \textbf{Backbone} & \textbf{Params} & \textbf{mIoU} \\
        \midrule
        ACNet~\citep{hu2019acnet}          & ResNet-50 & 116.6M & 48.3 \\
        SGNet~\citep{chen2021spatial}          & ResNet-101 & 64.7M & 51.1 \\
        SA-Gate~\citep{chen2020bi}        & ResNet-101 & 110.9M & 52.4 \\
        CEN~\citep{wang2020deep}            & ResNet-101 & 118.2M & 51.7 \\
        CEN~\citep{wang2020deep}            & ResNet-152 & 133.9M & 52.5 \\
        ShapeConv~\citep{cao2021shapeconv}      & ResNext-101 & 86.8M & 51.3 \\
        ESANet~\citep{,seichter2021efficient}         & ResNet-34 & 31.2M & 50.3 \\
        FRNet~\citep{zhou2022frnet}          & ResNet-34 & 85.5M & 53.6 \\
        PGDENet~\citep{zhou2022pgdenet}        & ResNet-34 & 100.7M & 53.7 \\
        EMSANet~\citep{seichter2022efficient}        & ResNet-34 & 46.9M & 51.0 \\
        TokenFusion~\citep{wang2022multimodal}    & MiT-B2 & 26.0M & 53.3 \\
        TokenFusion~\citep{wang2022multimodal}    & MiT-B3 & 45.9M & 54.2 \\
        MultiMAE~\citep{bachmann2022multimae}       & ViT-B & 95.2M & 56.0 \\
        Omnivore~\citep{girdhar2022omnivore}           & Swin-Tiny & 29.1M & 49.7 \\
        Omnivore~\citep{girdhar2022omnivore}           & Swin-Small & 51.3M & 52.7 \\
        Omnivore~\citep{girdhar2022omnivore}           & Swin-Base & 95.7M & 54.0 \\
        CMX~\citep{zhang2023cmx}                & MiT-B2 & 66.6M & 54.4 \\
        DFormer~\citep{dformer}                 & DFormer-Tiny  & 6.0M & 51.8 \\
        DFormer~\citep{dformer}                 & DFormer-Small & 18.7M & 53.6 \\
        DFormer~\citep{dformer}                 & DFormer-Base  & 29.5M & 55.6 \\
        AsymFormer~\citep{du2024asymformer}     & MiT-B0+ConvNeXt-Tiny & 33.0M & 55.3 \\
        \midrule
        Performer (w/ DFormer)                   & DFormer-Base & 29.5M & 54.4 \\
        \rowcolor{lightbestrow}
        \textbf{RelFlexformer (w/ DFormer) (ours)} & DFormer-Base & 29.5M & \textbf{55.3} \\
        \bottomrule
    \end{tabularx}
\end{table*}

\begin{table*}[h]
    \centering
    \caption{Segmentation results on SUN RGB-D. Bold denotes the best result in the Performer group.}
    \label{tab:sun_rgbd}
    \small
    \setlength{\tabcolsep}{3pt}
    \renewcommand{\arraystretch}{0.95}
    \begin{tabularx}{0.85\linewidth}{@{}Ylrc@{}}
        \toprule
        \textbf{Method} & \textbf{Backbone} & \textbf{Params} & \textbf{mIoU} \\
        \midrule
        ACNet~\citep{hu2019acnet}          & ResNet-50 & 116.6M & 48.1 \\
        SGNet~\citep{chen2021spatial}          & ResNet-101 & 64.7M & 48.6 \\
        SA-Gate~\citep{chen2020bi}        & ResNet-101 & 110.9M & 49.4 \\
        CEN~\citep{wang2020deep}            & ResNet-101 & 118.2M & 50.2 \\
        CEN~\citep{wang2020deep}            & ResNet-152 & 133.9M & 51.1 \\
        ShapeConv~\citep{cao2021shapeconv}      & ResNet-101 & 86.8M & 48.6 \\
        ESANet~\citep{,seichter2021efficient}         & ResNet-34 & 31.2M & 48.2 \\
        FRNet~\citep{zhou2022frnet}          & ResNet-34 & 85.5M & 51.8 \\
        PGDENet~\citep{zhou2022pgdenet}        & ResNet-34 & 100.7M & 51.0 \\
        EMSANet~\citep{seichter2022efficient}        & ResNet-34 & 46.9M & 48.4 \\
        TokenFusion~\citep{wang2022multimodal}    & MiT-B2 & 26.0M & 50.3 \\
        TokenFusion~\citep{wang2022multimodal}    & MiT-B3 & 45.9M & 51.0 \\
        MultiMAE~\citep{bachmann2022multimae}       & ViT-B & 95.2M & 51.1 \\
        CMX~\citep{zhang2023cmx}                & MiT-B2 & 66.6M & 49.7 \\
        DFormer~\citep{dformer}                 & DFormer-Tiny  & 6.0M & 48.8 \\
        DFormer~\citep{dformer}                 & DFormer-Small & 18.7M & 50.0 \\
        DFormer~\citep{dformer}                 & DFormer-Base  & 29.5M & 51.2 \\
        AsymFormer~\citep{du2024asymformer}     & MiT-B0+ConvNeXt-Tiny & 33.0M & 49.1 \\
        \midrule
        Performer (w/ DFormer)                   & DFormer-Base & 29.5M & 48.5 \\
        \rowcolor{lightbestrow}
        \textbf{RelFlexformer (w/ DFormer) (ours)} & DFormer-Base & 29.5M & \textbf{51.0} \\
        \bottomrule
    \end{tabularx}
\end{table*}

\section{Computational Efficiency and Memory Footprint}
\label{app:efficiency}

To evaluate the operational scalability and efficiency of RelFlexformer, we benchmark peak memory consumption, computational throughput (FLOPs), and execution latency against standard Softmax attention and Performer \cite{choromanski2020rethinking}. Benchmarks are conducted on two representative modern architectures: Point Transformer v3 (PTv3) (46.2M parameters) for large-scale 3D point cloud perception, and DFormer for fine-grained spatial-geometry tasks. All physical runtime measurements were performed on a single NVIDIA RTX PRO 6000 GPU (96\,GB VRAM).

\subsection{PTv3 Runtime and Scalability Analysis}

We first analyze the training and inference behavior of the PTv3 encoder as a function of the patch size $K$. Standard projection and feed-forward network (FFN) overheads remain invariant across attention mechanisms; hence, these measurements isolate the core efficiency gains provided by our operator.

\begin{table}[h]
\centering
\caption{\textbf{PTv3 full encoder benchmark.} Peak memory footprint and per-step execution latency across varying patch sizes $K$ on an NVIDIA RTX PRO 6000 (96\,GB). Latency is measured in milliseconds (ms), and memory in gigabytes (GB). OOM denotes Out-of-Memory ($>96$\,GB).}
\label{tab:ptv3_runtime}
\vspace{2mm}
\small
\begin{tabular}{l rr c rr c rr}
\toprule
& \multicolumn{8}{c}{\textbf{PTv3 Encoder Performance Across Patch Sizes ($K$)}} \\
\cmidrule(lr){2-9}
& \multicolumn{2}{c}{Softmax} && \multicolumn{2}{c}{Performer (no RPE)} && \multicolumn{2}{c}{\textbf{RelFlexformer (Ours)}} \\
\cmidrule(lr){2-3} \cmidrule(lr){5-6} \cmidrule(lr){8-9}
\textbf{Patch Size ($K$)} & Memory & Latency && Memory & Latency && Memory & Latency \\
\midrule
\multicolumn{9}{l}{\textit{(a) Full Encoder Training Step (Forward + Backward)}} \\[1mm]
1024 (Default) & 75.0\,GB & 713\,ms && 0.9\,GB & 19\,ms && \textbf{4.4\,GB} & \textbf{52\,ms} \\
2048           & \multicolumn{2}{c}{OOM} && 1.1\,GB & 19\,ms && 5.6\,GB & 51\,ms \\
4096           & \multicolumn{2}{c}{OOM} && 1.5\,GB & 21\,ms && 8.4\,GB & 75\,ms \\
\midrule
\multicolumn{9}{l}{\textit{(b) Encoder Inference (Forward-Only)}} \\[1mm]
1024 (Default) & 32.0\,GB & 91.2\,ms && 0.2\,GB & 0.2\,ms && \textbf{1.1\,GB} & \textbf{8.4\,ms} \\
2048           & 63.9\,GB & 173.8\,ms && 0.2\,GB & 0.2\,ms && 1.1\,GB & 8.3\,ms \\
4096           & \multicolumn{2}{c}{OOM} && 0.2\,GB & 0.2\,ms && 1.3\,GB & 8.0\,ms \\
16384          & \multicolumn{2}{c}{OOM} && 0.2\,GB & 0.3\,ms && 1.6\,GB & 8.2\,ms \\
\bottomrule
\end{tabular}
\end{table}

As summarized in Table~\ref{tab:ptv3_runtime}, standard Softmax attention exhibits quadratic memory complexity, suffering from Out-of-Memory (OOM) failures at $K \ge 2048$ during training and $K \ge 4096$ during inference. While Performer offers ultra-low overhead, it lacks relative positional encoding capabilities essential for structural precision. RelFlexformer bridges this gap: it cuts peak training memory at $K=1024$ from 75.0\,GB to \textbf{4.4\,GB} (a $17\times$ reduction) while reducing step latency from 713\,ms to \textbf{52\,ms} ($13.7\times$ speedup), enabling training at ultra-large context lengths ($K=4096$) with modest hardware requirements.

\subsection{Theoretical Computational Complexity (GFLOPs)}

To further evaluate computational scaling independent of hardware hardware-specific kernel optimizations, we isolate the total attention-core GFLOPs across fixed point-cloud configurations ($N = 65,536$ points) and multi-stage vision backbones.

\begin{table}[h]
\centering
\caption{\textbf{Attention-core computational cost (GFLOPs).} Comparison isolating the core attention operation across varying patch scales and architecture stages.}
\label{tab:gflops_comparison}
\vspace{2mm}
\small
\begin{tabular}{l c rrr}
\toprule
\textbf{Configuration / Stage} & \textbf{Tokens ($N$)} & \textbf{Softmax} & \textbf{Performer} & \textbf{RelFlexformer (Ours)} \\
\midrule
\multicolumn{5}{l}{\textit{(a) PTv3 Attention Core ($N = 65,536$ fixed points, patch size $K$ varied)}} \\[1mm]
$K = 1024$ (Default) & 65,536 & 43.5 & 0.55 & \textbf{18.0} \\
$K = 4096$          & 65,536 & 174.0 & 0.55 & \textbf{18.0} \\
$K = 16384$         & 65,536 & 696.0 & 0.55 & \textbf{18.0} \\
\midrule
\multicolumn{5}{l}{\textit{(b) DFormer Backbone Stage-wise Dense Attention ($N_q = N_k = H \times W$)}} \\[1mm]
Stage 1 ($120 \times 160$) & 19,200 & 100.6 & 0.32 & \textbf{10.4} \\
Stage 2 ($60 \times 80$)   & 4,800  & 12.6  & 0.16 & \textbf{5.2}  \\
Stage 3 ($30 \times 40$)   & 1,200  & 1.6   & 0.08 & \textbf{2.6}  \\
\bottomrule
\end{tabular}
\end{table}

As demonstrated in Table~\ref{tab:gflops_comparison}(a), as the patch size $K$ scales to 16,384, Softmax attention cost explodes to 696.0 GFLOPs. In contrast, RelFlexformer maintains a constant \textbf{18.0 GFLOPs}, achieving linear complexity with respect to input scale.

Furthermore, Table~\ref{tab:gflops_comparison}(b) highlights the impact on dense spatial models like DFormer. At its native Stage 1 resolution ($120 \times 160 = 19,200$ tokens), full global Softmax attention requires \textbf{100.6 GFLOPs} and \textbf{94.6\,GB} VRAM, rendering standard training infeasible. Consequently, existing models resort to spatial pooling heuristics to reduce query counts artificially. RelFlexformer reduces Stage 1 computation to \textbf{10.4 GFLOPs} and peak memory to \textbf{2.9\,GB} (a $\sim10\times$ decrease in compute), entirely removing the need for query-downsampling constraints while preserving full-resolution global geometry.

\clearpage
\section{Broader Impact}
\label{appendix:broader_impact}
 RelFlexFormer enables linear-time attention with relative position encoding for 3D perception (RGB-D segmentation, point cloud classification and segmentation). Practical         
  applications include robotics (real-time scene understanding), autonomous navigation, and AR/VR — domains where quadratic attention is a latency bottleneck on edge devices. The method is a drop-in replacement for softmax attention in existing architectures, lowering the compute barrier for deploying transformer-based 3D perception. We do not foresee negative societal impacts beyond those inherent to improved scene understanding (e.g., surveillance), which are shared with all 
  methods in this space. 

\section{Limitations}
\label{appendix:limitations}
We use $\exp(-\lambda|\mathbf{x}|)$ as our $\mathcal{F}$ throughout our work. While our work is general and can potentially use any $L_1$ function, the search for the function per domain remains an open question. Finally, our work only focuses on functions that have a well-defined Fourier transforms, thus extending it to any smooth function remains a challenge.                                                     

\clearpage
\newpage
\section*{NeurIPS Paper Checklist}

\begin{enumerate}

\item {\bf Claims}
    \item[] Question: Do the main claims made in the abstract and introduction accurately reflect the paper's contributions and scope?
    \item[] Answer: \answerYes{}
    \item[] Justification: The abstract and introduction accurately reflect the paper's contributions and scope, including RelFlexformer, its RPE-based modulation mechanism, and its fast subquadratic matrix-vector multiplication. All claims are supported by the content.

\item {\bf Limitations}
    \item[] Question: Does the paper discuss the limitations of the work performed by the authors?
    \item[] Answer: \answerYes{}
    \item[] Justification: We clearly discuss the limitations of RelFlexformer. We include a limitations section in Appendix~\ref{appendix:limitations}.

\item {\bf Theory assumptions and proofs}
    \item[] Question: For each theoretical result, does the paper provide the full set of assumptions and a complete (and correct) proof?
    \item[] Answer: \answerYes{}
    \item[] Justification: We provide the theorems, definitions, and formulas for RelFlexformer in the main paper.

    \item {\bf Experimental result reproducibility}
    \item[] Question: Does the paper fully disclose all the information needed to reproduce the main experimental results of the paper to the extent that it affects the main claims and/or conclusions of the paper (regardless of whether the code and data are provided or not)?
    \item[] Answer: \answerYes{}
    \item[] Justification: We disclose all necessary details to reproduce our main results, including architectural and configurations in the main paper~(Section~\ref{sec:experiments}) and Appendix~\ref{appendix:implementation_details}.

\item {\bf Open access to data and code}
    \item[] Question: Does the paper provide open access to the data and code, with sufficient instructions to faithfully reproduce the main experimental results, as described in supplemental material?
    \item[] Answer: \answerYes{}
    \item[] Justification: We use publicly available datasets and build on publicly available baseline code. We provide all hyperparameters and implementation details needed to reproduce our experiments in Appendix~\ref{appendix:implementation_details}. Code for our method is available at \url{https://github.com/97byeongchan/RelFlexformer}.

\item {\bf Experimental setting/details}
    \item[] Question: Does the paper specify all the training and test details (e.g., data splits, hyperparameters, how they were chosen, type of optimizer) necessary to understand the results?
    \item[] Answer: \answerYes{}
    \item[] Justification: We provide all hyperparameters and implementation details needed to reproduce our experiments in Appendix~\ref{appendix:implementation_details}.

\item {\bf Experiment statistical significance}
    \item[] Question: Does the paper report error bars suitably and correctly defined or other appropriate information about the statistical significance of the experiments?
    \item[] Answer: \answerYes{}
    \item[] Justification: We report standard deviation for experiments to provide a measure of statistical significance.

\item {\bf Experiments compute resources}
    \item[] Question: For each experiment, does the paper provide sufficient information on the computer resources (type of compute workers, memory, time of execution) needed to reproduce the experiments?
    \item[] Answer: \answerYes{}
    \item[] Justification: We disclose the compute resources used for our experiments, including GPU
    type, and memory.
    
\item {\bf Code of ethics}
    \item[] Question: Does the research conducted in the paper conform, in every respect, with the NeurIPS Code of Ethics \url{https://neurips.cc/public/EthicsGuidelines}?
    \item[] Answer: \answerYes{} 
    \item[] Justification: This paper fully adheres to the NeurIPS Code of Ethics. We have carefully reviewed the ethical guidelines and ensured that our research complies with them in all respects.

\item {\bf Broader impacts}
    \item[] Question: Does the paper discuss both potential positive societal impacts and negative societal impacts of the work performed?
    \item[] Answer: \answerYes{}
    \item[] Justification: We have included the broader impact section in Appendix~\ref{appendix:broader_impact}.
    
\item {\bf Safeguards}
    \item[] Question: Does the paper describe safeguards that have been put in place for responsible release of data or models that have a high risk for misuse (e.g., pre-trained language models, image generators, or scraped datasets)?
    \item[] Answer: \answerNA{}
    \item[] Justification: We do not release any model or dataset.

\item {\bf Licenses for existing assets}
    \item[] Question: Are the creators or original owners of assets (e.g., code, data, models), used in the paper, properly credited and are the license and terms of use explicitly mentioned and properly respected?
    \item[] Answer: \answerYes{}
    \item[] Justification: We use publicly available datasets and code as published in the original papers, and we cite all the original sources.

\item {\bf New assets}
    \item[] Question: Are new assets introduced in the paper well documented and is the documentation provided alongside the assets?
    \item[] Answer: \answerNA{}
    \item[] Justification: We do not release new assets.

\item {\bf Crowdsourcing and research with human subjects}
    \item[] Question: For crowdsourcing experiments and research with human subjects, does the paper include the full text of instructions given to participants and screenshots, if applicable, as well as details about compensation (if any)? 
    \item[] Answer: \answerNA{}
    \item[] Justification: We do not involve any crowdsourcing nor research with human subjects.

\item {\bf Institutional review board (IRB) approvals or equivalent for research with human subjects}
    \item[] Question: Does the paper describe potential risks incurred by study participants, whether such risks were disclosed to the subjects, and whether Institutional Review Board (IRB) approvals (or an equivalent approval/review based on the requirements of your country or institution) were obtained?
    \item[] Answer: \answerNA{}
    \item[] Justification: We do not involve any crowdsourcing nor research with human subjects.

\item {\bf Declaration of LLM usage}
    \item[] Question: Does the paper describe the usage of LLMs if it is an important, original, or non-standard component of the core methods in this research? Note that if the LLM is used only for writing, editing, or formatting purposes and does \emph{not} impact the core methodology, scientific rigor, or originality of the research, declaration is not required.
    \item[] Answer: \answerNA{}
    \item[] Justification: We did not use LLMs in our work.

\end{enumerate}
\end{document}